\documentclass[10pt, conference, letterpaper]{IEEEtran}
\IEEEoverridecommandlockouts
\pdfoutput=1
\usepackage{cite}
\usepackage{amsmath,amssymb,amsfonts}
\usepackage{algorithmic}
\usepackage{graphicx}
\usepackage{textcomp}
\usepackage{comment}
\usepackage{autobreak}
\usepackage{xcolor}
\usepackage{array}
\usepackage{amsmath,amssymb,amsfonts}
\usepackage{amsthm}
\usepackage{mathrsfs}
\usepackage{CJK}
\usepackage[ruled,linesnumbered]{algorithm2e}
\usepackage{amsmath}
\usepackage{url}
\usepackage{color}
\usepackage{pifont}
\usepackage{braket}
\usepackage{multirow}
\usepackage{subfigure}

\newtheorem{assumption}{Assumption}
\newtheorem{lemma}{Lemma}
\newtheorem{theorem}{Theorem}
\newtheorem{definition}{Definition}

\newtheorem{proposition}{Proposition}

\allowdisplaybreaks[4]

\ifCLASSINFOpdf
\else
\fi
\begin{document}
%
% paper title
% Titles are generally capitalized except for words such as a, an, and, as,
% at, but, by, for, in, nor, of, on, or, the, to and up, which are usually
% not capitalized unless they are the first or last word of the title.
% Linebreaks \\ can be used within to get better formatting as desired.
% Do not put math or special symbols in the title.
\title{Fed-CVLC: Compressing Federated Learning Communications with Variable-Length Codes}

\author{\IEEEauthorblockN{Xiaoxin Su\IEEEauthorrefmark{2}, Yipeng Zhou\IEEEauthorrefmark{3}, Laizhong Cui\IEEEauthorrefmark{1} \IEEEauthorrefmark{2}\IEEEauthorrefmark{4}, John C.S. Lui\IEEEauthorrefmark{5} and Jiangchuan Liu\IEEEauthorrefmark{6}}
\IEEEauthorblockA{\IEEEauthorrefmark{2}College of Computer Science and Software Engineering, Shenzhen University, Shenzhen, China}
\IEEEauthorblockA{\IEEEauthorrefmark{3}School  of Computing, Faculty of Science and Engineering, Macquarie University, Sydney, Australia}
\IEEEauthorblockA{\IEEEauthorrefmark{4}Guangdong Laboratory of Artificial Intelligence and Digital Economy (SZ), Shenzhen University, Shenzhen, China}
\IEEEauthorblockA{\IEEEauthorrefmark{5}Department of Computer Science and Engineering, The Chinese University of Hong Kong, HKSAR}
\IEEEauthorblockA{\IEEEauthorrefmark{6}School of Computing Science, Simon Fraser University, Canada}
\IEEEauthorblockA{Email: suxiaoxin2016@163.com, yipeng.zhou@mq.edu.au, cuilz@szu.edu.cn, cslui@cse.cuhk.edu.hk and jcliu@sfu.ca}

\IEEEauthorblockA{
	\thanks{
		\newline This work has been partially supported by National Key Research and Development Plan of China under Grant No. 2022YFB3102302, National Natural Science Foundation of China under Grant No. U23B2026 and No. 62372305, Shenzhen Science and Technology Program under Grant No. RCYX20200714114645048, the RGC GRF 14202923 and an NSERC Discovery Grant.}
		\thanks{\textit{\IEEEauthorrefmark{1}Corresponding author: Laizhong Cui}
	}
}
}

% make the title area
\maketitle

% As a general rule, do not put math, special symbols or citations
% in the abstract
\begin{abstract}

In Federated Learning (FL) paradigm, a parameter server (PS) concurrently communicates with distributed participating clients for model collection,  update aggregation, and model distribution over multiple rounds, without touching  private data owned by  individual clients. FL is appealing in preserving data privacy; yet the communication between the PS and  scattered clients can be a severe bottleneck. Model compression algorithms, such as quantization and sparsification, have been suggested but they generally assume a fixed code length, which does not reflect the heterogeneity and variability of model updates. In this paper, through both analysis and experiments, we show strong evidences  that variable-length is beneficial for compression in FL. We accordingly present Fed-CVLC (Federated Learning Compression with Variable-Length Codes), which fine-tunes the code length  in response of the dynamics of model updates. We develop optimal tuning strategy that  minimizes the loss function (equivalent to maximizing the model utility) subject to the budget for communication. We further demonstrate that Fed-CVLC is indeed a general compression design that bridges quantization and sparsification, with greater flexibility.  %{{{XXXXXX ??? In addition, Fed-CVLC is friendly for implementation  by fixing the code length for model updates encapsulated into the same communication packet. XXXX}}} 
Extensive experiments have been conducted with public datasets %\emph{i.e.}, FEMNIST, CIFAR-10 and CIFAR-100, which 
to demonstrate that Fed-CVLC remarkably outperforms state-of-the-art baselines, improving model utility by 1.50\%-5.44\%, or shrinking communication traffic by 16.67\%-41.61\%.
\end{abstract}

% \begin{IEEEkeywords}
% %Federated Learning, Compression Rate,  Communication Traffic
% \end{IEEEkeywords}

% no keywords

% For peer review papers, you can put extra information on the cover
% page as needed:
% \ifCLASSOPTIONpeerreview
% \begin{center} \bfseries EDICS Category: 3-BBND \end{center}
% \fi
%
% For peerreview papers, this IEEEtran command inserts a page break and
% creates the second title. It will be ignored for other modes.
\IEEEpeerreviewmaketitle

\section{Introduction}
\label{introduction}

Training advanced models, such as Convolutional Neural Network (CNN), needs a massive amount of training data. Collecting the data from clients however may expose their privacy with data breach or privacy infringement risks \cite{yang2019federated}.  To preserve data privacy while training models, the federated learning (FL) paradigm
was devised in \cite{mcmahan2017communication}, which orchestrates multiple clients to train a model together without exposing their raw data. FL has attracted tremendous attention from both industry and academia \cite{lim2020federated} due to its capability in utilizing distributed resources with privacy protection  \cite{8466361}.

In FL, data samples are owned and distributed on  multiple decentralized clients. A parameter server (PS) is deployed to coordinate the model training process for multiple  rounds via Internet communications.  Briefly speaking, there are three key steps in FL training \cite{mcmahan2017communication}. In \emph{Step 1},  the PS selects participating clients at the beginning of each communication round, a.k.a., a global iteration, to distribute the latest model; In \emph{Step 2}, each client receiving the model conducts local training iterations with local data samples to update the model; In \emph{Step 3}, model updates  will be returned by the participating clients to the PS, which will aggregate the collected model updates to revise the model before ending a communication round. A new communication round will then start from Step 1 again, until terminate conditions are met \cite{hamer2020fedboost}. 

Since the distributed clients and the PS communicate over the Internet, the training time of FL can be seriously prolonged in these steps. In particular, in Step 3,  the connection between some clients and the remote PS can be quite slow ($\sim$ 1Mbps) \cite{rothchild2020fetchsgd}, and thereby hindering the overall time efficiency. Modern Internet links are often asymmetric, with the uplink capacity being  much smaller than the downlink capacity \cite{konevcny2016federated},  exacerbating the communication bottleneck of FL. 
These communication challenges are particularly severe  when training advanced high-dimensional models, such as ResNet \cite{he2016deep} and Transformer \cite{NeurIPS2017_3f5ee243} with tens of millions of parameters.

To expedite FL, one may compress model updates to be transmitted from participating clients to the PS \emph{with the cost of less accurate model updates}. The compression algorithms for FL can be broadly classified as  {\em quantization} and {\em sparsification}, which commonly fix the code length.  The former  expresses each model update with a fewer number of bits. For example, in \cite{wen2017terngrad} and \cite{bernstein2018signsgd}, each model update is only represented by 2 bits and 1 bit, respectively, to achieve 16 and 32 compression rates by supposing that each original model update takes 32 bits.
Different from quantization,  sparsification  accelerates communications by transmitting a fewer number of model updates. For instance, the $Top_k$ compression algorithm only transmits  $k$ model updates of the largest magnitudes from clients to the PS. $Top_k$ can achieve a much higher compression rate than that of quantization in that most model updates are close to 0 with a very small magnitude in practice \cite{lin2017deep}. Recently, more sophisticated hybrid compression was proposed  \cite{sattler2019robust} by combining quantization and sparsification.  
%A more comprehensive comparison between our work and related work is presented in Appendix. \ref{RelatedWork}. A notable difference is that  all aforementioned works simply employed \emph{a fixed code length} to compress model updates.

The core problem in FL model compression is how to minimize the accuracy loss of model updates due to compression. Using the fixed code length for model compression fails to exploit the fact revealed by existing works \cite{lin2017deep, m2021efficient} that the distribution  of the magnitudes of model updates is very skewed, and thus existing works cannot fully minimize the compression loss. In view of this gap,  we propose a novel Federated Learning Compression with Variable-Length Codes (Fed-CVLC) algorithm, which employs codes of different lengths to compress model updates. To motivate our study, we use a simple example to  illustrate the idea of variable-length codes for model compression. %{\color{blue}by exhibiting the distribution of model updates.} %a practical experiment first.  
Then, the optimization problem to minimize the loss function subject to limited communication traffic is formulated with respect to the code length of each model update. Besides,  our proposed method is friendly for protocol implementation in that model updates encapsulated into the same communication packet are  constrained by the same code length to  avoid excessive  overhead. Based on our analysis, we show that the formulated problem can be efficiently solved to optimally determine the code length of each model update. 

The advantage of our design lies in that  Fed-CVLC unifies the design  of quantization and sparsification since quantization and sparsification can be regarded as special cases of Fed-CVLC. Fed-CVLC degenerates to quantization compression, if the code length is fixed as a constant for all model updates.  Similarly, Fed-CVLC degenerates to the $Top_k$ sparsification compression if we fix the code length for $k$ top model updates and set code length as 0 for remaining model updates. Therefore, quantization and sparsification are rigid in the sense that  they only set one or two code lengths for compression. In contrast, Fed-CVLC is more general and flexible by taking the  code length as a variable toward compression in FL.

At last, we  conduct extensive experiments with  CIFAR-10, FEMNIST and CIFAR-100 datasets  using Fed-CVLC and state-of-the-art  compression algorithms. The experimental results demonstrate the superb performance of our Fed-CVLC algorithm, which can improve model accuracy by 3.21\% and shrink communication traffic by 27.64\% on average in comparison with state-of-the-art compression baselines.

The rest of the paper is organized as follows. We introduce the related works in Sec.~\ref{RelatedWork} and 
preliminaries in Sec.~\ref{Preliminaries}. %and empirically  demonstrate that the performance  gain of variable-length codes  in Sec~.\ref{Motivation}. 
The detailed analysis of our problem  and the design of Fed-CVLC with a discussion of its implementation  are delivered in Sec.~\ref{Analysis}.
%presented in Sec.~\ref{AlgorithmImplementation}. 
Ultimately, we conclude our work in Sec.~\ref{Conclusion} after evaluating the performance of Fed-CVLC in Sec.~\ref{Experiment}.

\section{Related Work}
\label{RelatedWork}

In this section, we discuss related works from two perspectives: federated learning design and model compression. 

%\subsection{Federated Learning Design} 
Federated learning (FL) is a novel distributed machine learning framework, which was originally proposed by Google \cite{mcmahan2017communication} to protect mobile user privacy. FedAvg \cite{mcmahan2017communication} has become the most fundamental algorithm for conducting FL. FedAvg performs multiple local iterations in each round to  reduce communication frequency between clients and the PS. 
%. The communication efficiency was discussed in  as well. To diminish communication cost, FedAvg can  perform multiple local iterations to  reduce communication frequency between clients and the PS. 
The convergence rate of FedAvg was derived in \cite{li2019convergence} and \cite{yang2021achieving} to validate the effectiveness of FL in model training. 
%The convergence rate of FedAvg was derived in   \cite{li2019convergence} under strongly convex loss and  \cite{yang2021achieving} under non-convex loss, respectively, which validated the effectiveness of FL in model training. 
Various variants of FedAvg were devised to improve FL from different perspectives \cite{9796724, 9796719, 9796721}. 
{In \cite{9796818,9796935}, the authors discussed the effect of client sampling and designed client scheduling strategies to optimize the model performance. Wang \emph{et al.} \cite{9488756} designed a FL training paradigm with hierarchical aggregation and proposed an efficient algorithm to determine the optimal cluster structure with resource constraints.}
%The FedProx algorithm \cite{MLSYS2020_38af8613} was proposed to alleviate the influence of data heterogeneity on model utility in FL by introducing a proximal term to the local loss function. The SCAFFOLD algorithm \cite{karimireddy2020scaffold} was developed to solve the client drift problem caused by non-IID data distribution on FL clients.  SCAFFOLD introduced additional control variables into FedAvg to overcome the client drift problem, so that FL can converge faster. 
%To shorten the network delay in FL, \cite{zhu2021delayed} designed the DGA algorithm, which can achieve parallelism in both computation and communication by delaying gradient averaging. 

Model compression is  perpendicular to the above works, which significantly improves the training efficiency of FL by largely shrinking communication traffic. We briefly discuss two most popular  model compression approaches: quantization and sparsification. 

The essence of quantization  is to use a fewer number of bits to represent each model update. %It  cannot achieve a very high compression rate because no parameter will be discarded from communications.
PQ \cite{suresh2017distributed} and QSGD \cite{alistarh2017qsgd} are efficient quantization algorithms that map model updates to a finite set of discrete values in an unbiased manner, thereby reducing the number of bits for each update.
%QSGD \cite{alistarh2017qsgd} was devised as an efficient unbiased compression algorithm, which  adjusts the number of bits to compress model updates. %by weighing the communication bandwidth and convergence time. 
%SignSGD was proposed by \cite{bernstein2018signsgd} to alleviate communication bottlenecks by only transmitting the sign of each model update. Its performance was theoretically guaranteed by convergence analysis.  
%PQ algorithm was designed in \cite{suresh2017distributed}, which maps model updates  in each interval to their upper/lower bounds in an unbiased manner. 
DAdaQuant was designed by \cite{honig2022dadaquant} using time- and client- adaption to adjust quantization levels, optimizing the final model utility of FL. 

Sparsification, achieving a much  higher compression rate, is more aggressive by removing unimportant model updates from communications.  
\cite{wangni2018gradient} proposed to drop a proportion of model parameters stochastically and scale the remaining parameters. 
The $Top_k$ algorithm \cite{stich2018sparsified, qian2021error} only transmits $k$ model updates of the largest magnitudes. An error compensation mechanism was used to guarantee convergence performance. 
%\cite{NEURIPS2019_d202ed5b, sattler2019robust}  proposed novel compress trick to combine  quantization and sparsification. 
\cite{NEURIPS2019_d202ed5b, sattler2019robust} came up with novel compression algorithms by combining  quantization and sparsification. 
The convergence of these algorithms was theoretically analyzed. 

Nonetheless, all these quantization and sparsification algorithms fix the code length for model compression regardless  the skewness of the magnitude distribution of  model updates. Different from existing works, Fed-CVLC adopts  variable-length codes to better gauge the importance of model updates by using  more bits to represent  larger magnitudes and fewer bits to represent smaller magnitudes. %which however has not been explored yet. 

%Existing compression algorithms all use the same number of bits for each parameter, yet the importance of different parameters varies greatly from one to another. Therefore, representing each parameter equally cannot fully utilize the given traffic. We will explore the effect of different parameters on model convergence in this work and implement an optimal parameter quantization strategy based on this analysis.

\section{Preliminaries}
\label{Preliminaries}

Consider a generic FL system with $N>1$ clients in the system. Local data samples on each client $i$ are drawn from a local data distribution denoted by $\mathcal{D}_i$, where $i\in[N]$. The objective of FL is to train a model vector $\mathbf{w}\in\mathbb{R}^d$ that minimizes the global loss function $F(\mathbf{w})$ defined across all $N$ clients \cite{li2019convergence}. In other words,  $\mathbf{w}^*=\arg\min_{\mathbf{w}} \left[F(\mathbf{w})=\frac{1}{N}\sum_{i=1}^N\mathbb{E}_{\xi\sim\mathcal{D}_i}f(\xi,\mathbf{w})\right]$,
%as follows:
%\begin{align}
%    \mathbf{w}^*=\arg\min_{\mathbf{w}} \left[F(\mathbf{w})=\frac{1}{N}\sum_{i=1}^N\mathbb{E}_{\xi\sim\mathcal{D}_i}f(\xi,\mathbf{w})\right],
%    \label{EQ:FLObj}
%\end{align}
where $\xi$ is a particular data sample randomly drawn from the local data distribution $\mathcal{D}_i$ and $f()$ is the loss function evaluating the model using a particular sample.

Let $F_i(\mathbf{w})=\mathbb{E}_{\xi\sim\mathcal{D}_i}f(\xi,\mathbf{w})$ denote the local loss function on client $i$. FedAvg \cite{mcmahan2017communication}, the most fundamental model training algorithm in FL,   and its variants \cite{wang2019adaptive,luo2020cost} commonly have the  following three steps for deriving $\mathbf{w}^*$. % Eq.~\eqref{EQ:FLObj}.

\begin{enumerate}%[itemsep=0.5pt,topsep=0.5pt,parsep=0.5pt, leftmargin =15pt ]
	\item At the beginning of communication round $t$, the PS distributes the latest vector $\mathbf{w}_t$ to a number of selected clients, denoted by set $\mathcal{S}_t$. 
	\item  Each selected client $i$ will use downloaded $\mathbf{w}_t$  to initialize her local training vector as  $\mathbf{w}^i_{t,0}=\mathbf{w}_t$ and update the model locally for $E\ge1$ local iterations %using the stochastic gradient descent (SGD) algorithm 
 $\mathbf{w}^i_{t,j+1}=\mathbf{w}^i_{t,j}-\eta_{t,j}\nabla F_i(\mathbf{w}^i_{t,j},\mathcal{B}^i_{t,j})\ $ for $j=0,\dots, E-1.$
	Here $\mathcal{B}^i_{t,j}$ is a data sample  batch (with size $|\mathcal{B}^i_{t,j}|$) randomly selected from client $i$'s local dataset for the $j$-th local iteration. % in communication round $t$. 
     Model updates updated by client $i$ in communication round $t$ are denoted by $\mathbf{U}^i_{t}=\mathbf{w}^i_{t,0}-\mathbf{w}^i_{t,E}$, which will be returned to the PS.
	% Model updates updated by client $i$ in communication round $t$ are denoted by $\mathbf{U}^i_{t}$, which actually are  the change of model parameters, \emph{i.e.},  $\mathbf{U}^i_{t}=\mathbf{w}^i_{t,E}-\mathbf{w}^i_{t,0}.$
	% $\mathbf{U}^i_{t}$ will be returned to the PS.  
	\item The PS receiving $\mathbf{U}^i_{t}$'s from all participating clients will update the model by aggregating model updates, \emph{i.e.},   $\mathbf{w}_{t+1}=\mathbf{w}_t-\frac{1}{|\mathcal{S}_t|}\sum_{i\in\mathcal{S}_t}\mathbf{U}^i_{t}$.
	Then, the PS goes back to Step 1 until total $T\ge1$ iterations are executed. 
	
\end{enumerate}

%a model in FL, which consists of multiple global iterations. In the $t$-th iteration, the parameter server (PS) randomly selects $K$ clients (defined as $\mathcal{S}_t$) to participate in the training. 
To reduce the communication traffic for transmitting $\mathbf{U}_i^t$, various compression algorithms can be applied to compress $ \mathbf{U}_i^t$ into $ \hat{\mathbf{U}}_i^t$ which can consume much less communication bandwidth. The discrepancy between  $ \mathbf{U}_i^t$ and $ \hat{\mathbf{U}}_i^t$ is called the compression error, a.k.a., the accuracy loss of  model updates  attributed to model updates compression.

\section{Analysis and Algorithm Design}
\label{Analysis}

In this section, we introduce a simple example to illustrate the intuition of our design, and then formulate the problem to minimize the compression error  with  a fixed compression rate using variable-length codes. %Based on  convergence analysis, we show that this problem is  equivalent to minimizing the training loss of FL and accordingly present an efficient solution.  The
At last, we discuss how to implement our algorithm with lightweight overhead. 

{\color{black}
\subsection{A Motivation Example}
State-of-the-art quantization and sparsification compression algorithms generally set the fixed code length to compress  model updates, which unfortunately may not minimize the compression error of  model updates with different values. For example,  suppose we use fixed 6 bits to encode two numbers with values 1,000 and 2, respectively. Encoding the number 1,000 then loses accuracy significantly, whereas encoding the number 2  wastes bits. This observation inspires us to consider  encoding model updates with variable-length codes.  

We use an example to illustrate the intuition of our algorithm.  In Fig.~\ref{fig:top_quan}, we plot the distribution of model updates by training a CNN model with 300,000 parameters. Here, the x-axis represents the ranked ID of model updates (in a log scale because most model updates are close to $0$),  the left y-axis represents the magnitude of model updates and the right y-axis represents the number of bits to encode each model update. The continuous curve represents the magnitude of model updates, while each vertical bar  represents the number of bits to encode its corresponding model update.

%Note that we only plot top 500 model updates due to that the rest model updates are very close to $0$. 
From Fig.~\ref{fig:top_quan}, we can observe that the distribution of model update magnitudes is very skewed with only a small number of model updates far away from $0$. In Fig.~\ref{fig:top_quan}(a), quantization uniformly compresses each model update with 6 bits regardless of the magnitude. Sparsification plus quantization plotted in Fig.~\ref{fig:top_quan}(b) is more advanced than quantization by using 6 bits to  encode top 3,000 model updates (\emph{i.e.}, top 1\%) and 0 bits to encode the rest small model updates. The principle of our algorithm is presented in Fig.~\ref{fig:top_quan}(c), which is  more flexible in setting the code length in accordance with the model update magnitude for  compression. % each model update. %{\color{blue} Since the x-axis is taken logarithmically, the rank of the model update is not uniformly distributed. Although (b) and (c) in Fig.~\ref{fig:top_quan} seem to upload most of the parameters, they actually upload only 1\% of the parameters.}

\begin{figure}[h!]
\vspace{-0.45cm}
	\setlength{\abovecaptionskip}{-0.1cm}
	\centering
	\includegraphics[width=\linewidth]{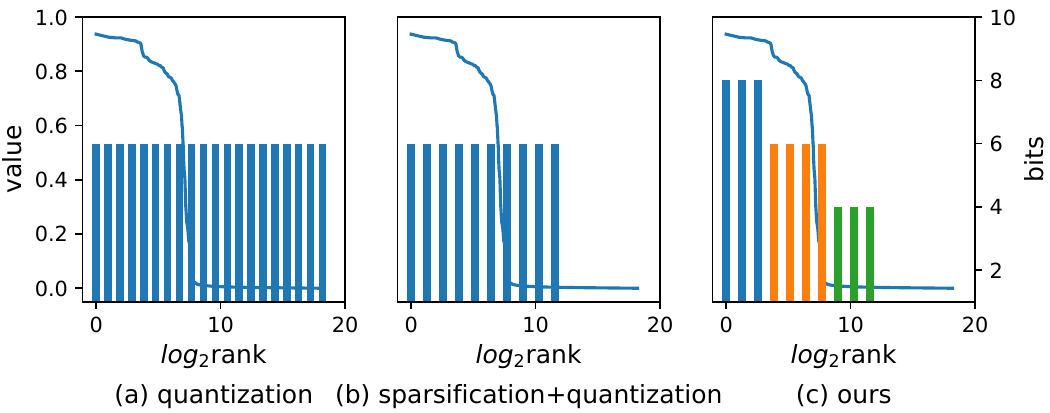}
	\caption{The distribution of  model updates in a particular  round and the code length of each model update set by different compression algorithms.%: {\color{blue}quantization(a), sparsification+quantization(b), ours(c)}. %Comparison of model accuracy on the test set achieved by three different compression algorithms (d).
 }
	\label{fig:top_quan}
 \vspace{-0.5cm}
\end{figure}

}

\subsection{Compression Error}
The core problem of model compression when shrinking communication traffic lies in incurred compression error causing inaccuracy of model updates. To quantify the impact of a model compression algorithm on model utility, it is imperative to explicitly analyze and quantify the compression error.

To make  analysis of compression algorithms formulable, it is common to use a distribution to model the magnitudes of model updates. For example, in \cite{stich2018sparsified}, it was assumed that the distribution of model update magnitudes is uniform. 
However, this simple uniform distribution cannot capture the skewness  of  model updates in practice. In \cite{m2021efficient}, a more generic power law distribution was adopted to better represent the distribution of model update magnitudes. 

\begin{definition}
	\label{DEF:PowerDist}
	On client $i$, model updates ranked in a descending order of their absolute values satisfy a power law decreasing, \emph{i.e.,} $|\mathbf{U}^i_t\{l\}|\le \phi_il^{\alpha_i} \quad \forall l\in\{1,2,\dots,d\}$, 
	%\begin{align}
	%   |\mathbf{U}^i_t\{l\}|\le \phi_il^{-\alpha_i} \quad \forall l\in\{1,2,\dots,d\},
	%\end{align}
	where $\mathbf{U}^i_t\{l\}$ is the $l$-th largest model update (in terms of absolute value) in $\mathbf{U}^i_t$, $\alpha_i<0$ is the decay exponent controlling the decaying rate of the distribution and $\phi_i$ is a constant.
\end{definition}
% In our design, there are three operations to compress model updates in $\mathbf{U}_t^i$. 
There are three operations to compress model updates $\mathbf{U}_t^i$.
\begin{itemize}%[itemsep=0.5pt,topsep=0.5pt,parsep=0.5pt, leftmargin =15pt ]
	\item {\bf Sparsification.} First, sparsification is applied to remove model updates of a small magnitude yielding $\mathbf{\overline{U}}_t^i$.
	\item {\bf Quantization.} Second, unbiased quantization is applied to further compress remaining top model updates with a fixed number of bits. $\widetilde{\mathbf{U}}^i_t$ is model updates after unbiased quantization satisfying $\mathbb{E}[\widetilde{\mathbf{U}}^i_t]=\mathbf{\overline{U}}_t^i$. 
	\item {\bf Scaling.} Third,  $\widetilde{\mathbf{U}}^i_t$ should be scaled down by a constant $B$ to theoretically guarantee the convergence of FL with compressed model updates. In other words, the scaling operation is $\hat{\mathbf{U}}^i_t=\frac{\widetilde{\mathbf{U}}^i_t}{B}$. 
\end{itemize}
%it was required that %$\mathbb{E}\|\mathbf{U}^i_t-\hat{\mathbf{U}}^i_t\|\le \gamma_i\|\mathbf{U}^i_t\|^2$, where $0<\gamma_i<1)$ . Thus,
Let $\hat{\mathbf{U}}^i_t$ denote final compressed model updates after three operations. Each operation will incur compression error. The error attributed to quantization and scaling operations is quantified as follows.

{\color{black}
	\begin{lemma}
		\label{LEM:QuanError}
		Suppose a model update vector $\mathbf{\overline{U}}$ with $z_y$ elements and each element is quantified by $y$ bits %in an unbiased manner to obtain the vector
		$\widetilde{\mathbf{U}}$. %Each element in the vector $\widetilde{\mathbf{U}}$ is divided by the constant $B$ to obtain $\hat{\mathbf{U}}$. T
		After quantization and scaling operations,  the error between $\mathbf{\overline{U}}$ and $\hat{\mathbf{U}}$ is
		$ \mathbb{E}\|\hat{\mathbf{U}}-\mathbf{\overline{U}}\|^2\le \frac{Q(z_y,y)+(B-1)^2}{B^2}\|\mathbf{\overline{U}}\|^2$,
		% \begin{align}
			%    \mathbb{E}\|\hat{\mathbf{U}}-\mathbf{\overline{U}}\|^2\le \frac{Q(z_y,y)+(B-1)^2}{B^2}\|\mathbf{\overline{U}}\|^2,
			% \end{align}
		where $Q(z_y,y)$ is the quantization error to quantify  $z_y$ model updates with  each model update expressed by $y$ bits. 
	\end{lemma}
	The detailed proof is presented in Appendix~\ref{Proof_LEM:QuanError}.
}
%We define $\mathbf{U}^i_t$ as the model update and $\hat{\mathbf{U}}^i_t$ as the data to be transmitted after compression. 
%Since the transmitted data is power-law distributed, the importance of different parameters varies widely. Based on this idea, we should consider variable-length compression of model updates. Considering that the data transmitted in FL is divided by packets, we can specify the size of each parameter in the packet by packet header, thus enabling variable-length coding.
Note that our design is a generic framework in the sense that different unbiased quantization algorithms can be applied in our design. Thus, $Q(z_y, y)$ needs not to be fixed until we  specify some quantization algorithms.
In general, we can encode top $k$ model updates with $Y$ kinds of codes of different lengths. For example, if $Y=32$, quantization compression, \emph{e.g.}, QSGD, can compress a  model update with $32, 31, \dots,$ or $ 1$ bit. Intuitively speaking, a model update of a larger magnitude should be quantified with more bits. Thus, we put ranked top model updates into $Y$ groups.  Formally, we let $\mathcal{Z}_y$ with cardinality $z_y$ represent the set of model updates that will be compressed with codes of length $y$.

Considering that model updates will be grouped to compress with different code lengths, the compression error after three operations can be bounded as below. 
%it is first necessary to define the error of its compression algorithm. The compression error of our designed algorithm is presented as follows.
% \begin{lemma}
	% \label{LEM:ErrorBound}
	% 	Suppose that model updates are compressed by $Top_k$ sparsification and heterogeneous quantization of parameters across packages. The error of model updates cause by compression and packet losses is bounded by $\mathbb{E}\|\mathbf{U}^i_t-\hat{\mathbf{U}}^i_t\|\le \gamma_i\|\mathbf{U}^i_t\|^2$. %	\begin{equation}
		% 	%	\mathbb{E}\|\mathbf{U}^i_t-\hat{\mathbf{U}}^i_t\|\le \gamma_i\|\mathbf{U}^i_t\|^2.
		% 	%\end{equation}
		% Here,
		% \begin{align}
			% \label{EQ:ErrorBound}
			%     &\gamma_i=\frac{d^{2\alpha_i+1}-(k+1)^{2\alpha_i+1}}{d^{2\alpha_i+1}-1} + \frac{1}{d^{2\alpha_i+1}-1}\sum_{r=1}^R\notag\\
			%     &\Big((\frac{Q(k_r,b_r)}{B^2}+B^{'}) \big((\sum_{l=1}^r k_l)^{2\alpha_i+1} - (\sum_{l=1}^{r-1}k_l+1)^{2\alpha_i+1}\big) \Big),
			% \end{align}
		% where $k=\sum_{r=1}^Rk_r$, $Q(k_r,b_r)$ is compression error of any quantization algorithm with $k_r$ parameters and parameter size $b_r$ in packet $r$, $B$ satisfies $B>\frac{Q(k_r,b_r)+1}{2}$ and $B^{'}=\frac{(B-1)^2}{B}$ to ensure $0<\gamma_i<1$.
		% \end{lemma}

	{\color{black}
		\begin{proposition}
			\label{LEM:ErrorBound}
			On client $i$ in communication round $t$, the error of model updates after sparsification operation selecting top $k$ model updates, quantization operation with $Y$ kinds of variable-length codes and scaling operation with factor $B$ is bounded by $\mathbb{E}\|\mathbf{U}^i_t-\hat{\mathbf{U}}^i_t\|\le \gamma_i\|\mathbf{U}^i_t\|^2$.
			Here,
			\begin{align}
				\label{EQ:ErrorBound}
				&\gamma_i=\frac{d^{\beta_i}-(k+1)^{\beta_i}}{d^{\beta_i}-1} + \frac{1}{d^{\beta_i}-1}\sum_{y=1}^Y\notag\\
				&\big((\frac{Q(z_y,y)}{B^2}+\frac{1}{B^2_{c}}) \big(Z_y^{\beta_i} - Z_{y+1}^{\beta_i}\big) \big),
			\end{align}
			where $\beta_i=2\alpha_i+1$, $k=\sum_{y=1}^Yz_y$ and $Z_y =z_Y+z_{Y-1}+\dots+z_y $.
			$Q(z_y, y)$ is defined in Lemma~\ref{LEM:QuanError}.  To guarantee the convergence of FL, it is required that $0<\gamma_i<1$. 
			$B$ is set as a constant  scale factor  satisfying $B>\max_{\forall y}\frac{Q(z_y, y)+1}{2}$,
			%\begin{align}
			%\label{EQ:ConstantB}
			%    B>\frac{Q(\mathcal{Z}_y,y)+1}{2},
			%\end{align}
			to ensure the convergence of FL \footnote{The exact value of $B$ is $\max_{\forall y}Q(z_y, y)+1$, which is difficult for analysis. Thus, in each global iteration we use $\max_{\forall y}Q(z_y, y)+1$ of the previous iteration for approximation.}.
			$B_c$ is a constant satisfying $\frac{1}{B}+\frac{1}{B_c}=1$, which is used to simplify our presentation.
		\end{proposition}
		
		In Proposition~\ref{LEM:ErrorBound}, $\gamma_i$ is a critical parameter. A smaller $\gamma_i$ implies a smaller compression error and hence higher model utility. 
		Proposition~\ref{LEM:ErrorBound}  is proved by combining the compression error for model updates in $Y$ different groups. %For each group, the error has two components, \emph{i.e.}, sparsification error and quantization error. 
		The detailed proof is presented in Appendix \ref{Proof_LEM:ErrorBound}.
		
		% Consider an ideal situation, where PS automatically recognizes the size of each parameter without the need to rely on packets. In this case we define the existence of $Y$ different bit sizes $b_y (y=1,2,\dots,Y)$. We assume that there are $k_y$ parameters with bit size $b_y$. The compression error in this case is:
		% \begin{align}
			%     &\gamma_i=\frac{d^{2\alpha_i+1}-(k+1)^{2\alpha_i+1}}{d^{2\alpha_i+1}-1} + \frac{1}{d^{2\alpha_i+1}-1}\sum_{y=1}^Y\notag\\
			%     &\Big((\frac{Q(k_y,b_y)}{B^2}+B^{'}) \big((\sum_{l=1}^y k_l)^{2\alpha_i+1} - (\sum_{l=1}^{y-1}k_l+1)^{2\alpha_i+1}\big) \Big),
			% \end{align}
	}
	
	%We analyze the compression error by dividing it into two parts: sparsification error and heterogeneous quantization error. For the specific derivation process, refer to Appendix~\ref{Proof_LEM:ErrorBound}.

Our problem lies in how to optimally determine the number of model updates $z_y$ that should be encoded into $y$ bits, \emph{e.g.,} $y=1,\dots,32$.
	%The generic value of our approach will be illustrated  via convergence rate analysis in the next subsection before exploring how to optimize $z_y$. 

	%This problem can be transformed to the problem how to locally minimizing the compression error $\gamma_i$ on client $i$ by tuning $z_y$. Its effectiveness is discussed in Appendix~\ref{ConvergenceRateofFL}, in which we demonstrate that minimizing $\gamma_i$  is equivalent to minimizing the loss function based on convergence rates derived by existing works for the $Top_k$ algorithm in FL. 

{\color{black} To demonstrate the generic merit of our approach, we list the convergence rates of the biased compression algorithms with compression error $\gamma_i$ proposed in existing works in Table~\ref{tab:convergence}. Note that we have ignored all constants that are not asymptotically relevant with the compression error and only show asymptotic convergence rates in Table~\ref{tab:convergence}. These works derived convergence rates in different expressions because they are slightly different in  assumptions of loss functions and setting of the learning rates. 
    %The detailed proof can be found from references listed in  Table~\ref{tab:convergence}, which has been omitted in our work to avoid repetition. 
    
    \begin{table*}[h]
    \setlength{\abovecaptionskip}{-0.05cm}
    	\centering
    	\caption{Convergence of different loss functions in federated learning using different learning rates under compressed communication: $F^*$ is the minimum value of the global loss function $F()$, $\mathbf{w}_T$ is the aggregation of models from different clients, $\nabla F(\mathbf{w})$ is the global gradient of model $\mathbf{w}$, $\mathbf{v}_1$ is a random variable drawn from the sequence $\{\mathbf{w}_{t}\}$ with the learning rate as the weight, $\mathbf{v}_2$ is a random variable drawn from the sequence of $\{\mathbf{w}^i_{t,j}\}$ uniformly, $\epsilon_i$ is a constant satisfying $\gamma_i<\epsilon_i<1$.}
    	\begin{tabular}{|c|c|c|c|}
    		\hline
    		Algorithm & Loss Function & Learning Rate & Asymptotic Convergence Rate \\
    		\hline         
    		Qsparse-local-SGD \cite{NEURIPS2019_d202ed5b} & strongly convex & $\mathcal{O}(\frac{1}{t})$ & $F(\mathbf{w}_T)-F^*\le\mathcal{O}(\frac{1}{T}+\sum_{i=1}^N\frac{\gamma_i(1+\gamma_i)}{T^2})$ \\
    		%BCFL \cite{9917527} & non-convex & $\mathcal{O}(\frac{1}{\sqrt{T}})$ & $\|\nabla F(\mathbf{z}_1)\|^2\le\mathcal{O}(\frac{1}{\sqrt{T}}+\sum_{i=1}^N\frac{1}{(1-\gamma_i)^2T})$  \\
    		CFedAvg \cite{9589061} & non-convex & $\mathcal{O}(\frac{1}{\sqrt{t}})$ & $\|\nabla F(\mathbf{w}_T)\|^2\le\mathcal{O}(\sum_{i=1}^N\frac{\gamma_i}{(\epsilon_i-\gamma_i)\sqrt{T}})$ \\
    		CFedAvg \cite{9589061} & non-convex & $\mathcal{O}(\frac{1}{\sqrt{T}})$ & $\|\nabla F(\mathbf{v}_1)\|^2\le\mathcal{O}(\sum_{i=1}^N\frac{\gamma_i}{(\epsilon_i-\gamma_i)\sqrt{T}}+\frac{1}{T})$ \\
    		FT-LSGD-DB \cite{li2020talk} & non-convex & $\mathcal{O}(\frac{1}{\sqrt{T}})$ & $\|\nabla F(\mathbf{v}_2)\|^2\le\mathcal{O}(\frac{1}{\sqrt{T}}+\sum_{i=1}^N\frac{1}{(1-\gamma_i)^2T})$ \\
    		\hline
    	\end{tabular}
    	\label{tab:convergence}
     \vspace{-2.0em}
    \end{table*}
    
%    The $Top_k$ algorithm has been widely studied by existing works with derived convergence rates. Actually, 
Based on convergence rates listed in Table~\ref{tab:convergence}, we can conclude that our approach is effective in improving convergence and hence model utility for all these works by lowering the compression error without compromising the compression rate. %From convergence rates listed in Table~\ref{tab:convergence}, 
More specifically, we  draw the following conclusions.
    \begin{itemize}
        \item Lowering $\gamma_i$ for a given traffic limitation can always make the gap to the convergence point smaller, and hence lower the loss function and improve model utility.  
        \item $\gamma_i$'s are independent with each other implying that each client can independently make her local optimization of $\gamma_i$ so as to optimize the global model utility.
    \end{itemize}
    % \begin{observation}
    % 	In all convergence rates in Table~\ref{tab:convergence}, one can verify that:
    % 	\begin{itemize}
    % 		\item Lowering $\gamma_i$ for a given traffic limitation can always make the gap to the convergence point smaller, and hence lower the loss function and improve model utility.  
    % 		\item $\gamma_i$'s are independent with each other implying that each client can independently make her local optimization of $\gamma_i$ so as to optimize the global model utility.
    % 	\end{itemize}
    % \end{observation}
    Based on the above observations, our problem becomes how to locally minimize the compression error $\gamma_i$ on client $i$ by tuning $z_y$'s.
    }
    To make our discussion concise, we only discuss how to minimize $\gamma$ for an arbitrarily selected client hereafter.

	\subsection{Optimizing Code Length}
	Before we formulate the problem to optimize the code length for each model update, we define the constraints of model transmission in FL.

	%\subsubsection{Constraint of Each Packet}
	By only transmitting top model updates, it is necessary to transmit the position ID of a selected top model update together with its value in FL. 
	It is well-known that data transmitted via Internet communications will be partitioned and encapsulated into packets.
	Each  packet transmitted in IP based networks has a header and a payload. The header part contains the necessary network information to decode the packet. The payload part contains position IDs and values of model updates. Suppose that there are $k$ top model updates that will be transmitted via $R$ packets.  In practice, the size of  each packet is limited and denoted by $b_r$. 
	Let $\mathcal{P}_r$ with cardinality $P_r$ denote the set of model updates that will be encapsulated into packet $r$ where $1\leq r\leq R$. 
	
	Let $u$ denote a particular model update. 
	By slightly abusing our notations, let $y_u$ denote the code length to express model update $u$. Due to the constraint of the packet size, we have
	\begin{equation}
		\label{EQ:PCons}
		\sum_{u\in \mathcal{P}_r}(s+y_u) +H\leq b_r, \quad  \textit{for $r=1, \dots, R$}.
	\end{equation}
	Here $H$ is the header size of packets and $s$ is the size of each position ID. If there are  total $d$ model updates in $\mathbf{U}_t^i$, we have $ s=\lceil\log_2{d}\rceil$. Equation~\eqref{EQ:PCons} indicates that it is impracticable to encapsulate all model updates (for training advanced high-dimensional models) into a single packet due to the limited size of each packet.

	{
		From \eqref{EQ:PCons}, we can observe that the communication efficiency is higher if the header part with size $H$  is  smaller relative to the packet size $b_r$. 
		However, it is rather difficult to 
		straightly minimize the size of the header part because it is dependent on the way how model updates are encapsulated into packets. When using variable-length codes, it is possible that model updates of different code lengths are encoded into the same packet. It is necessary to include sufficient information into the header part so that the PS can correctly understand the information contained in a packet. 
		For example,  if there are 10 model updates with code length 32, 20 model updates with code length 30, 40 model updates with code length 28. 
		Information of this complicated code method must be contained in the header part, which can significantly expands the size of $H$. To avoid impairing communication efficacy, we force that all model updates encapsulated into a particular packet $r$ must have the same code length $y_r$. With this constraint, the header part can be extremely simple since PS only need $s$ and $y_r$ to decode the payload of  packet $r$.

		Meanwhile, $H$ is relatively smaller if $b_r$ is bigger. 
		For simplicity, we let $b_r=b$ for all packets where $b$ is the maximum size of each packet. For instance, $b=1,500$ bytes for a typical TCP packet \cite{islam2016quality}. Our problem is converted to optimizing the code length for each packet, which implies that  the constraint in \eqref{EQ:PCons} is changed as
		\begin{align}
			\label{EQ:packetConstrain}
			P_r(s+y_r) +H\leq b, \quad  \textit{for $r=1, \dots, R$}
		\end{align}
		
		%When multiple lengths of parameters are included in a packet, the packet header needs to be specified for each length. The header needs to contain the value of each parameter length and the number of parameters with same length. In addition, the packet header needs to provide the necessary information to calculate the centroids for each parameter length. The information contained in the package header in this case can be very large. Encoding before transmission and decoding after reception is complicated by the need to constantly identify the length of the currently read parameter.
	}
	
	%Discuss that the overhead is heavy due to the large header size $H$ for the most general case. 
	
	%We further constrain that model updates encapsulated into the same packet has the same code length. 
	
	%{\color{blue}
		%The header of each package is of equal size, denoted by $H$.The packet header in this scenario only stores the information under one length, which greatly reduces the content of information. Once the parameter length is set according to the packet header information during encoding and decoding, it can be read continuously without considering additional information. Therefore the difficulty of implementation will also be greatly reduced. Based on this analysis, we will define the same length of parameters in each packet in the practical deployment.
		%}
	
	%\subsubsection{Constraint of Uplink  Traffic}
	
	The communication resource is limited, which should be further constrained so that FL will not spend excessive time on the communication of model updates. 
	Let $bR$ denote the budget of uplink communication traffic of a participating client in a communication round. $bR$ is determined by the uplink communication capacity of the client and the total training time budget of FL. Since we focus on the compression algorithm design, $bR$ is simply regarded as a constant in our problem, which has been specified before FL is conducted.
	% The communication resource is limited, which should be further constrained so that FL will not spend excessive time on the communication of model updates. 
	% Let $C$ denote the budget of uplink communication traffic of a participating client in a communication round. $C$ is determined by the uplink communication capacity of the client and the total training time budget of FL. Since we focus on the compression algorithm design, $C$ is simply regarded as a constant in our problem, which has been specified before FL is conducted.  Thus, the total number of packets that can be uploaded by an individual client in FL is constrained by 
	% \begin{align}
		% \label{EQ:UpCons}
		%     b R \leq C.
		% \end{align}
Note that the communication cost is also subject to the total number of communication rounds. However, our design can improve existing compression algorithms in FL for every communication round. Thus, the improvement of our design is regardless of the number of conducted communication rounds. 
		%The upper bound for a given communication traffic means a fixed communication time for each iteration. Therefore, the total number of iterations $T$ for convergence in Table~\ref{tab:convergence} is constant. All we need to consider in this situation is how to minimize the error of compression $\gamma_i$ by variable-length coding.

	%\subsubsection{Optimization Problem Formulation}
We proceed to formulate the problem to optimize the code length for each packet. 
In Proposition~\ref{LEM:ErrorBound}, the constraint of communication packets is not considered. However, the compression error bound will be revised if the code length must be identical for  model updates in the same packet. We put $k$ top model updates ranked by a descending order of their magnitudes into $R$ groups, which will be further encapsulated into $R$ packets. The compression error will be specified by: 
		%Consider that each packet contains parameters of one length only. Therefore we use $\mathcal{Z}_r$ with cardinality $P_r$ to denote the set of model updates in the $r$-th packet, which will be compressed to a length of $y_r$.
		\begin{proposition}
			\label{LEM:ErrorBoundR}
			Suppose top $k$ model updates from Client $i$'s $\mathbf{U}^i_t$ at the $t$-th communication round are selected and encapsulated into $R$ packets. The set of model updates in packet $r$ is denoted by $\mathcal{P}_r$, each of which is encoded with $y_r$ bits. After three compression operations, $\gamma_i$ indicating the compression error on client $i$ is 
			\begin{align}
				\label{EQ:ErrorBoundR}
				&\gamma_i=\frac{d^{\beta_i}-(k+1)^{\beta_i}}{d^{\beta_i}-1} + \frac{1}{d^{\beta_i}-1}\sum_{r=1}^R\notag\\
				&\Big((\frac{Q(P_r,y_r)}{B^2}+\frac{1}{B^2_{c}}) \big(Z_r^{\beta_i} - Z_{r-1}^{\beta_i}\big) \Big),
			\end{align}
			where  $k=\sum_{r=1}^RP_r$, $Z_r =P_1+\dots+P_r $, $Q(P_r,y_r)$ represents the quantization error of model updates in $\mathcal{P}_r$ if each model update is expressed by $y_r$ bits  and $B>\max_{\forall r}\frac{Q(P_r, y_r)+1}{2}$.
			Here, the definitions of $\beta_i$ and $B_c$ are the same as those in Proposition~\ref{LEM:ErrorBound}.
			%is a scale factor that can be regarded as a constant satisfying , to make $0<\gamma_i<1$ ensuring the convergence of FL. is a constant satisfying $\frac{1}{B}+\frac{1}{B_c}=1$, which is used to simplify our presentation.
		\end{proposition}
		The proof of Proposition~\ref{LEM:ErrorBoundR} is similar to that of Proposition~\ref{LEM:ErrorBound}.
		Considering that the code length $y_r$  is determined by $P_r$, \emph{i.e.}, $y_r = \frac{b-H}{P_r}-s$, we use the term $Q(P_r)$ in lieu of $Q(P_r, y_r)$ for simplicity, hereafter.
		
		%are very close in form, but they represent different physical meanings. Proposition~\ref{LEM:ErrorBound} is grouped by bit size, but it fails to communicate efficiently. Proposition~\ref{LEM:ErrorBoundR} is grouped by packet, with the same length of parameters in each packet, which facilitates practical implementation.
		Our optimization problem is to tune the number of model updates in each packet $P_1, P_2,\dots, P_R$ to minimize $\gamma_i$ defined in \eqref{EQ:ErrorBoundR}. %  while satisfying communication constraints.   
		Formally, we have  
		%satisfying \eqref{EQ:packetConstrain}.
	\begin{align}
		\label{EQ:OptProb}
		&\mathbb{P}1:  \min_{P_1,P_2,\dots,P_R,k}\quad \gamma_i, \notag\\
		\qquad s.t. \!\!\!\!\!\!\!\!	& \qquad  \eqref{EQ:packetConstrain},  \eqref{EQ:ErrorBoundR} , \sum_{r=1}^{R}P_r =k, 1\le k\le d.
	\end{align}
	{\color{black}
		It is non-trivial to solve problem $\mathbb{P}1$ because: 1) the expression of $\gamma_i$  is very complicated for analyzing its convex property; 2) the search space is too large given $R+1$ variables, \emph{i.e.}, $P_1, \dots, P_R, k$.  
		
		To solve $\mathbb{P}1$, we  add two more constraints to reduce the search space. First, a model update of a larger magnitude should be compressed by a code with more bits, and vice verse. Given that the size of each packet is identical and model updates are encapsulated into $R$ packets by a descending order of their absolute values, it implies that 
		\begin{align}
			\label{EQ:ConsP}
			P_1\leq P_2\leq\dots \leq P_R. 
		\end{align}
		Second, the search space of $k$ is constrained by a lower bound value $k_{min} = R$, implying single model update in each packet, and a upper bound value $k_{max}=\frac{R(b-H)}{s+1}$, implying that each model update is represented by a single bit. %Here $s= \lceil \log_2{d}\rceil$ is the number of bits occupied by each position ID. 
		Considering more constraints, our problem is formulated as 
		\begin{align}
			%\label{EQ:OptProb2}
			&\mathbb{P}2:  \min_{P_1,P_2,\dots,P_R,k}\quad \gamma_i, \notag\\
			\qquad s.t. \!\!\!\!\!\!\!\!	& \qquad  \eqref{EQ:packetConstrain},  \eqref{EQ:ErrorBoundR}, \eqref{EQ:ConsP},  \sum_{r=1}^{R}P_r =k, k_{min}\le k\le k_{max}.\notag
		\end{align}
		Problem $\mathbb{P}2$ can be solved by two steps. 
		
		\emph{Step 1.} Given the range of $k$, enumerate all values of $k$. 
  
		\emph{Step 2.} Once $k$ is fixed by step 1, tune $P_1, P_2, \dots, P_R$ to minimize $\gamma_i$. 
		% \begin{itemize}
			%     \item \emph{Step 1.} Given the limited range of $k$, enumerate every possible value of $k$. 
			%     \item \emph{Step 2.} Once $k$ is fixed by step 1, tune $P_1, P_2, \dots, P_R$ to minimize $\gamma_i$. 
			% \end{itemize}
		
		In step 2, $P_1, P_2, \dots, P_R$ can be tuned by a series of atomic  operations. 
		Before we define the atomic operation, let us consider a special case with only two consecutive packets, \emph{i.e.,} $r-1$ and $r$, to carry a fixed number of total $X$ model updates. For this special case, we simplify the objective $\gamma_i$ by discarding all irrelevant constants %, such as $d^\beta_i-1$,
		to  get the function $f(X, P_{r-1}, P_r)$, subjecting to $P_{r-1}+ P_r =X$, as follows.  
		\begin{align}
			\label{EQ:LossFuncF}
			&f(X, P_{r-1}, P_r)=\Big(\frac{Q(P_r)}{B^2}+\frac{1}{B_c^2}\Big)\frac{Z_{r}^{\beta_i}-Z_{r-1}^{\beta_i}}{d^{\beta_i}-1}\notag\\
			&\qquad + \Big(\frac{Q(P_{r-1})}{B^2}+\frac{1}{B_c^2}\Big) \frac{Z_{r-1}^{\beta_i}-Z_{r-2}^{\beta_i}}{d^{\beta_i}-1}.
		\end{align}
		%where $Z_r=P_1+\dots+P_r$.%, $y_r=\frac{b-H}{P_r}-s$ and $y_{r-1}=\frac{b-H}{P_{r-1}}-s$. 
		\begin{definition}
			Atomic Operation: Given $f(X, P_{r-1}, P_r)$, subjecting to $P_{r-1}+P_r =X$, the atomic operation can tune $P_{r}$ such that $f(X, P_{r-1}, P_r)$ is minimized. 
		\end{definition}
		The atomic operation can be efficiently carried out because of the convexity of $f(X, P_{r-1}, P_r)$.
		\begin{lemma}
			\label{LEM:ConvexOfF}
			If $Q(P_r)$ is an increasing convex function with respect to $P_r$, $f(X, P_{r-1}, P_r)$, subjecting to $P_{r-1}+P_r =X$, is a strongly convex function with respect to $P_{r}$.
		\end{lemma}
		The proof is presented in Appendix~\ref{Proof_LEM:ConvexOfF}. 
		Although expressions of compression errors are dependent on the specific designs of different quantization algorithms, we can observe  their common property that $Q(P_r)$ is an increasing convex function with respect to $P_r$. Intuitively speaking, if $P_r$ is increased, it implies that more model updates are compressed with a fixed number of bits. Consequently, the compression error is enlarged. Besides, with the increase of $P_r$, the growth rate of the compression error will become larger and larger.  We will use PQ and QSGD, two typical quantization algorithms, to verify this in Appendix~\ref{CaseStudy}.
		
		%Later on, we will use PQ and QSGD, two typical quantization algorithms, to validate this point.\footnote{Due to limited space, it is difficult to enumerate all possible quantization algorithms.}    

		\begin{theorem}
			\label{THE:OptimalSolution}
			When $k$ is fixed, if there exists $P_1^*, \dots, P_R^*$ such that $f(X, P_{r-1}^*, P_r^*)$ is minimized for $2\le r\le R$,  $P_1^*, \dots, P_R^*$ can minimize $\gamma_i$ in $\mathbb{P}2$ and  $P_1^*\leq \dots \leq  P_R^*$.
		\end{theorem}
		The proof is presented in Appendix~\ref{Proof_THE:OptimalSolution}. $P_1^*, \dots, P_R^*$ can be found by Sequential Minimal Optimization (SMO) \cite{platt1998sequential}. SMO was originally designed to optimize a support vector machine (SVM). It breaks an original optimization problem into multiple subproblems by only optimizing  two variables each time. After repeatedly selecting and optimizing two variables, it approaches the objective of the SVM. Inspired by SMO, we call atomic operations to optimize the number of model updates in any two consecutive packets for multiple times until all $f$'s get converged implying that we get $P_1^*, \dots, P_R^*$ for a given $k$. Finally, the optimal $k$ is selected by the one, denoted by $k^*$,  minimizing $\gamma_i$.

		\begin{comment}
			involves too many variables. There are exponents of the sum of variables in \eqref{EQ:ErrorBoundR}, which is difficult to analyze in such a form. Therefore, we cannot optimize $\mathbb{P}1$ directly and need to design a more efficient method.
			
			The optimization variable $k$ of $\mathbb{P}1$ represents the sparsification error in $\gamma_i$, and $P_1, P_2, \dots, P_R$ represent the variable-length quantization error. Therefore, we correspondingly divide \eqref{EQ:ErrorBoundR} into two parts for optimization. For the sparsification error, we iterate through different $k$. After fixing $k$, we analyze how to divide these $k$ parameters into $R$ packages.
			
			\begin{assumption}
				\label{Assump:ParasAssign}
				When assigning parameters, a package will only split its own parameters to the next package.
			\end{assumption}
			The assumption is reasonable. We divide the parameters into $R$ packages after sorting them by absolute value. Therefore the magnitude of the parameters in each package is decreasing. When parameters in the first package to be assigned to other packages, they can only be assigned to the second package. The other packages are similar.
			
			\begin{theorem}
				\label{THE:OptimalAssign}
				Let Assump.~\ref{Assump:ParasAssign} hold and given $k$ parameters, we obtain the optimal assignment of $k$ parameters to $R$ packages by continuously iterating the optimal division of parameters between two adjacent packages.
			\end{theorem}
			The detailed proof is presented in Appendix~\ref{Proof_THE:OptimalAssign}.
		\end{comment}

	}

	{\color{black}

	}

	\subsection{Implementation Considerations}
	\label{AlgorithmImplementation}

    \begin{algorithm}[t]
        \label{Fed-CVLCAlgorithm}
        \caption{Optimizing Code Length in Fed-CVLC}
        \LinesNumbered
        % \KwIn{model updates $\mathbf{U}^i_t$}
        % \KwOut{Compressed model updates $\hat{\mathbf{U}}^i_t$}
        Input model updates $\mathbf{U}^i_t$ to compute $\alpha_i$ in Definition~\ref{DEF:PowerDist} and initialize $\gamma_i^*\leftarrow\infty$.\\
        \For{$k=k_{min}$ {\bfseries to} $k_{max}$}{
            $P_1,\dots, P_R\leftarrow\frac{k}{R}$.\\
            \Repeat{Convergence}{
                \For{$r=2$ {\bfseries to} $R$}{
                    $P_{r-1},P_{r}\leftarrow$ Optimize \eqref{EQ:LossFuncF}.\\
                    $y_{r-1},y_r\leftarrow \frac{b-H}{P_{r-1}}-s,\frac{b-H}{P_r}-s$.\\
                }
            }
            Substitute $k,P_1,\dots,P_R$ into \eqref{EQ:ErrorBoundR} to compute $\gamma_i$.\\
        	\If{$\gamma_i<\gamma_i^*$}{
        		$\gamma_i^*,k^*,P_1^*,\dots,P_R^*\leftarrow\gamma_i,k,P_1,\dots,P_R$.\\
                $y_1^*,\dots,y_R^*\leftarrow y_1,\dots,y_R$.\\
        	}
        }
        \For{$r=1$ {\bfseries to} $R$}{
            According to the format in Fig.~\ref{fig:packetHeader}, encapsulate $P^*_r$ model updates in $\mathcal{P}_r$, each of which is quantized by $y_r^*$ bits, in the $r$-th packet.\\
        }
        Remaining parameters are set to 0 to get $\hat{\mathbf{U}}^i_t$.
    \end{algorithm}
	%We discuss the implementation of Fed-CVLC in this section.
	
	%\subsection{Packet Format}
	
	%Based on our analysis, we design Fed-CVLC by taking the communication overhead into account. To be specific, we describe the format of a packet generated by Fed-CVLC in Appendix~\ref{PacketFormat}.

    {\color{black}
    We illustrate how to implement our algorithm by encapsulating  model updates into  packets. Fig. \ref{fig:packetHeader} presents the format of  packets generated by Fed-CVLC, in which %.~\ref{fig:packetHeader}, 
    we explain the fields related to Fed-CVLC in a packet as follows. 
    %We first analyze the design of the packet in the Fed-CVLC algorithm.
    \begin{figure}[htbp]
    \vspace{-0.4cm}
	\setlength{\abovecaptionskip}{-0.1cm}
    	\centering
    	\includegraphics[width=0.85\linewidth]{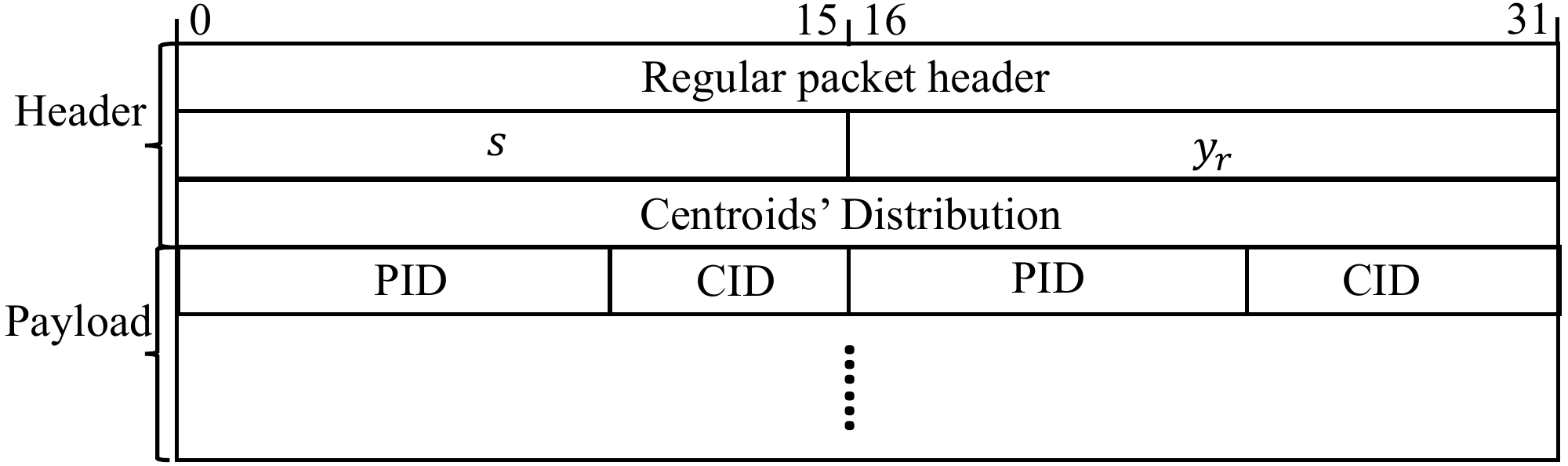}
    	\caption{Packet design when variable-length coding.}
    	\label{fig:packetHeader}
     \vspace{-0.4cm}
    \end{figure}

    \emph{1) }$s$ and $y_r$ indicate the number of bits to represent position ID and the number of bits for expressing each centriod.% value in the quantization.

    \emph{2) }Centroid Distribution helps the PS to compute  centroids of the quantization algorithm. %Different quantization algorithms take up different sizes. For example, t
	The centroids are uniformly distributed when using PQ and QSGD for quantization. The distribution of PQ only needs the maximum and minimum values of model updates (consuming 64 bits).  QSGD quantizes the absolute values of model updates, which needs to transmit the l2-norm of model updates (consuming 32 bits).

    \emph{3) }PID and CID indicate position ID of model update and centroid ID corresponding to the value of model update.
    
    % \begin{itemize}
    % 	\item $s$ in the Header part indicates the number of bits to represent the position ID of each model update;
    % 	\item $y_r$ in the Header part indicates the number of bits for expressing each centriod value in the quantization compression; 
    % 	\item Centroid Distribution in the Header part helps the PS to compute  centroid values of the quantization algorithm. Different quantization algorithms take up different sizes. For example, the centroids are uniformly distributed when using PQ and QSGD for quantization. The distribution of PQ only needs the maximum and minimum values of model updates (consuming 64 bits).  QSGD quantizes the absolute values of model updates, which needs to transmit the l2-norm of model updates (consuming 32 bits);
    %     \item PID and CID in the payload part indicate the position ID of the model update to be transmitted and the ID number of the centroid corresponding to the value of the model update, respectively.
    % 	% \item PID in the Payload part indicates the position ID of the model update to be transmitted; 
    % 	% \item CID in the Payload part indicates the ID number of the centroid corresponding to the value of the model update. 
    % \end{itemize}
In our design, the communication overhead is lightweight, which is incurred by the $s$ and $y_r$ fields in the Header part.
In Algo.~\ref{Fed-CVLCAlgorithm}, we present the pseudocode to optimize code length in Fed-CVLC. It can be embedded into existing FL algorithms such as FedAvg, conducted per communication round before uploading model updates. In particular,  line 6 is the key step to optimize the number of model updates contained in two consecutive packets. Its time complexity  is $\mathcal{O}(\frac{b*R^2}{\log_2d})$, which is much lower than that of model training. In addition, this workload can be  offloaded to the PS if a client with limited  capacity transmits  the  model update distribution to the PS, which can solve $\mathbb{P}2$ and return  $k^*$ and  $P_r^*$  for  compression.

    %For example, for an advanced CNN model with 1 billion parameters and more than 60,000 types of centroid values, the field sizes of  $s$  and $y_r$ are no more than $32$ bits and $16$ bits, respectively.
    }
	%can find that compared to the regular packet, Fed-CVLC only needs to add extra information of $s$, $y_r$ and Centroids in the packet header. $s$ indicates the size of the position ID and $y_r$ indicates the size of each parameter. Centroids indicates the information needed to compute the set of discrete values for the parameters mapping, e.g., 64 bits for the PQ algorithm \cite{suresh2017distributed} and 32 bits for the QSGD algorithm \cite{alistarh2017qsgd}. Therefore, the additional information added by Fed-CVLC is very small.Different packets can declare different parameter lengths in their own headers using $y_r$. The package load stores the position ID and value of each parameter, the former indicating the position of the parameter in the model updats and the latter indicating which value in the centroids the parameter maps to.
	%\subsection{Fed-CVLC Implementation}
	% \begin{algorithm}[htbp]
		% \caption{Variable-Length Compression}
		% \label{alg:VLCompression}
		% \begin{algorithmic}
			%    \STATE {\bfseries Input:} model updates $\mathbf{U}^i_t$, packet size $b$, packet number $R$.
			%    \STATE Use $\mathbf{U}^i_t$ to fit the power-law distribution parameter $\alpha_i$
			%    \STATE The $r$-th package quantifies the parameters based on $P_r$ and $y_r$ and the rest of the parameters are set to 0 to get compressed model updates $\hat{\mathbf{U}}^i_t$.
			%    \STATE \textbf{Return} $\hat{\mathbf{U}}^i_t$.
			% \end{algorithmic}
		% \end{algorithm}
	{\color{blue}%The pseudo code of the Fed-CVLC algorithm is presented in Appendix~\ref{Fed-CVLCAlgorithm}. 
	%Compared with FedAvg, the main contribution of our work is presented in Line 9 to Line 17. It is a three-level nested loop. 
	%Briefly speaking, the implementation of Fed-CVLC contains three-level nested loop.  %there is a three-level nested loop in Fed-CVLC. 
    %The first loop enumerates all possible values of $k$ in the range from $k_{min}$ to $k_{max}$, which is a small interval. For example, the size of a typical TCP packet is 1.5KB. If $s$ takes $24$ bits supposing that the model has 10 million parameters. The upper bound of $k$ is less than $R*1500*8/(24+1) =480R$. If we perform an equidistant spanning search of $k$, the computation complexity can be further reduced. The second loop is to enumerate all pairs of  consecutive packets. Given $R$ packets in total, there are $R-1$ pairs. In the third loop, for each pair of packets, we call the atomic operation to optimize the number of model updates encapsulated into each packet. Lemma~\ref{LEM:ConvexOfF} guarantees that the atomic operation can be completed instantly.% (with less than 10 loops usually). 
    }

	%the used times of SMO algorithm, which is less than 10 . The third loop is the process of optimizing $R$ packages with the SMO algorithm. Within the loop is the optimization of a strongly convex function. So 

\section{Experiment}
\label{Experiment}

% \begin{figure*}[htb]
	% \setlength{\abovecaptionskip}{0.1cm}
	% \setlength{\belowcaptionskip}{-0.cm}
	%     \centering
	%     \includegraphics[width=0.9\linewidth]{result/CIFAR10.eps}
	%     \caption{Comparison of model accuracy on the test set of CIFAR-10.}
	%     \label{fig:CIFAR10}
	% \end{figure*}

% \begin{figure}[htbp]
	% \setlength{\abovecaptionskip}{0.1cm}
	% \setlength{\belowcaptionskip}{-0.cm}
	%     \centering
	%     \includegraphics[width=0.9\linewidth]{result/FEMNIST.eps}
	%     \caption{Comparison of model accuracy on the test set of FEMNIST using PQ (left) and QSGD (right).}
	%     \label{fig:FEMNIST}
	% \end{figure}

% \begin{figure*}[htb]
	% \setlength{\abovecaptionskip}{0.1cm}
	% \setlength{\belowcaptionskip}{-0.cm}
	%     \centering
	%     \includegraphics[width=0.9\linewidth]{result/CIFAR100.eps}
	%     \caption{Comparison of model accuracy on the test set of CIFAR-100.}
	%     \label{fig:CIFAR100}
	% \end{figure*}
In this section, we evaluate the performance of Fed-CVLC using public standard datasets and state-of-the-art baselines. 

\subsection{Experimental Settings}
\noindent{\bf Datasets.} In our experiments, we employ three standard image datasets: CIFAR-10, CIFAR-100 and FEMNIST datasets.
For both CIFAR-10 and CIFAR-100, there are 50,000 images as the training set and 10,000 images as  the test set. Each image has 3*32*32 pixels. The CIFAR-10 dataset has 10 labels, while  the CIFAR-100 dataset has 100 labels. The FEMNIST dataset contains handwritten digital and letter images with 62 labels in total. Each image is  with a size of 28*28.

For CIFAR-10 and CIFAR-100, we randomly  assign the training set data to clients.  Each client will be assigned with 500 samples. By tuning the label assignment, we have both IID and non-IID (not independent and identically distributed) data distributions in our experiments. For the IID data distribution, samples assigned to each client  are randomly selected from entire datasets of CIFAR-10 and CIFAR-100, respectively. For the non-IID data distribution, samples assigned to each client are selected from a subset containing only 5 out of 10 labels in CIFAR-10 and  20 out of 100 labels in CIFAR-100. 
FEMNIST is naturally non-IID since 
each client only owns images generated by different users. Each client will have  300-400 FEMNIST samples.

\noindent{\bf Models.} The FL task for each image dataset is to train a CNN model for predicting image  labels correctly. 
%The models used for classifying these images are summarized in Table~\ref{tab:ExperimentalSetup}. 
{We train CNN models with 3-layer (number of parameters $3*10^5$), 4-layer (number of parameters $5*10^6$) and 2-layer (number of parameters $4.5*10^5$) to classify CIFAR-10,CIFAR-100 and FEMNIST images,  respectively.}
The structure of each layer  is Convolution-BatchNormalization-MaxPooling. 
{We constrain that the number of packets transmitted by each client per communication round is 10 for CIFAR-10 and FEMNIST, and 90 for CIFAR-100.}
%The ``Packet" column represents the number of packets transmitted by each client per communication round in our experiments. 
Based on the definition of the maximum transmission unit (MTU) in \cite{islam2016quality}, we set the size of each communication packet to be 1,500 bytes. 
Our trained CNN models are frequently used by prior works \cite{honig2022dadaquant,wang2020optimizing}.
%{\bf YP: any reference to support your models?} 
% \begin{table}[htb]
	%     \centering
	%     \caption{Dataset and model settings}
	%     \begin{tabular}{cccc}
		%         \hline
		%         Dataset & Model & Parameters & Packets  \\
		%         \hline
		%         CIFAR-10 &  3-layer CNN & $3*10^5$ & 10  \\
		%         CIFAR-100 & 4-layer CNN & $5*10^6$ & 90 \\
		%         FEMNIST & 2-layer CNN & $4.5*10^5$ & 10\\
		%         \hline
		%     \end{tabular}
	%     \label{tab:ExperimentalSetup}
	% \end{table}

\noindent{\bf System Settings.}
We set up 100 clients and a single PS. In each communication round, 10 out of 100 clients are randomly selected by the PS to participate in training. Each selected client will update models  with $E=5$ local iterations per communication round. We set the learning rate as 0.1 for CIFAR-10, and 0.05 for FEMNIST and CIFAR-100.
%{\bf YP: how you set learning rate? how many global iterations?} 

\noindent{\bf Baselines.}
As we have introduced, combining quantization and sparsification is the state-of-the-art compression technique in FL. We implement such compression algorithms by combining  PQ/QSGD \cite{suresh2017distributed, alistarh2017qsgd} and $Top_k$ \cite{stich2018sparsified}, the well-known sparsification compression algorithm, as our baselines. 
To be specific, other than the $Top_k$ algorithm, we  implement $P(Q)_6Top_k$, $P(Q)_{8}Top_k$ and $P(Q)_{10}Top_k$ by combining PQ (QSGD) with $Top_k$ where each top model update is quantized with 6, 8 or  10 bits, respectively.
%We also implement $QTop_k$ by combining QSGD with $Top_k$. 
{We evaluate Fed-CVLC in comparison with baselines from two perspectives: {\bf model accuracy} and {\bf communication traffic}.}
 \subsection{Experimental Results}
\noindent{\bf Comparing Model Accuracy.} 
\begin{figure*}
	\setlength{\abovecaptionskip}{-0.1cm}
    \centering
    \subfigure{
        \begin{minipage}[t]{\linewidth}
        \centering
        \includegraphics[width=0.9\linewidth]{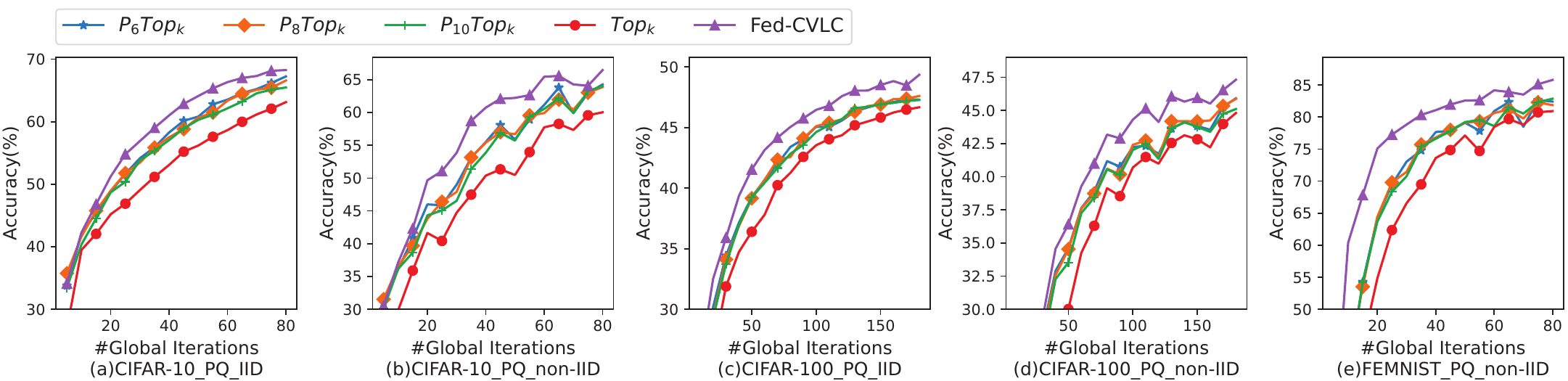}
        %\caption{fig1}
        \end{minipage}%
    }
    \\[-0.1cm]
    \subfigure{
        \begin{minipage}[t]{\linewidth}
        \centering
        \includegraphics[width=0.9\linewidth]{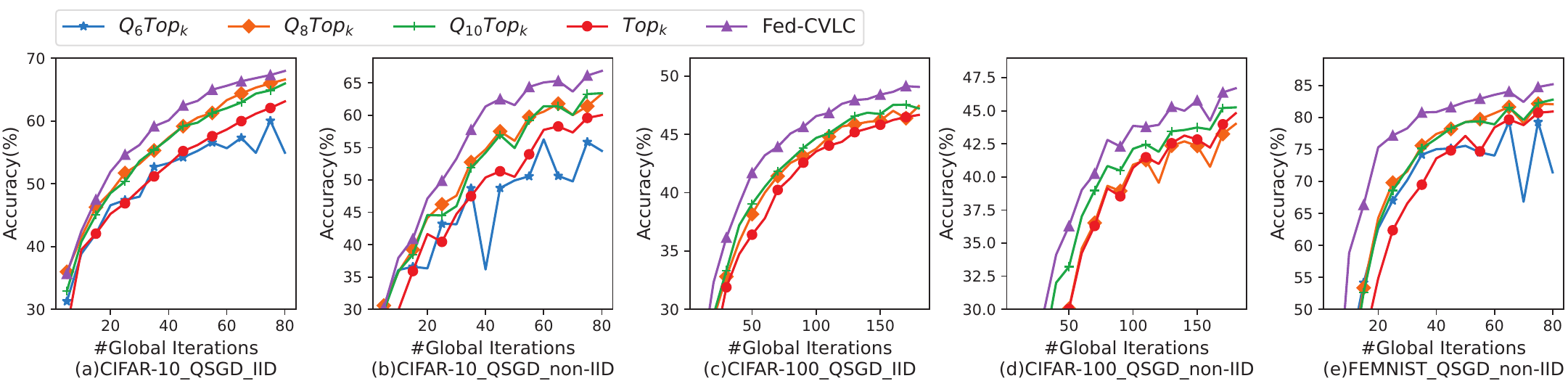}
        %\caption{fig1}
        \end{minipage}%
    }
    \caption{{\color{black}Comparison of model accuracy on the test set of CIFAR-10, CIFAR-100 and FEMNIST when using PQ(Top) and QSGD(Bottom) for quantization.}}
	\label{fig:PQ}
 \vspace{-0.4cm}
\end{figure*}
% \begin{figure*}[h!]
% %	\setlength{\abovecaptionskip}{0.1cm}
% %	\setlength{\belowcaptionskip}{-0.cm}
% 	\centering
% 	\includegraphics[width=0.9\linewidth]{pic/PQ.eps}
% 	\caption{Comparison of model accuracy on the test set of CIFAR-10, CIFAR-100 and FEMNIST.}
% 	\label{fig:PQ}
% \end{figure*}
% \begin{figure*}[h!]
% %	\setlength{\abovecaptionskip}{0.1cm}
% %	\setlength{\belowcaptionskip}{-0.cm}
% 	\centering
% 	\includegraphics[width=0.9\linewidth]{pic/QSGD.eps}
% 	\caption{Comparison of model accuracy on the test set of CIFAR-10, CIFAR-100 and FEMNIST.}
% 	\label{fig:QSGD}
% \end{figure*}
In 
Fig.~\ref{fig:PQ}, we compare the model accuracy of each algorithm  on different test datasets. 
%PQ and QSGD are used for quantization separately. 
In each figure, the x-axis represents the number of conducted communication rounds, while the y-axis represents the model accuracy on the test set. Note that we have fixed the amount of communication traffic (\emph{i.e.} 10 packets per client) for each compression algorithm per communication round so that we can fairly compare the model accuracy of these algorithms.  From experiment results presented in Fig.~\ref{fig:PQ}, we can draw the following observations: %1) Fed-CVLC can steadily achieve the highest model accuracy after every global iteration in all experiment cases under both IID and non-IID data distributions. The improvement of Fed-CVLC indicates the effectiveness to compress model updates with variable-length codes. 2) In most cases, combining PQ and  $Top_k$  can achieve better model accuracy than that of $Top_k$, indicating that the design combining quantization and sparsification is the state-of-the-art compression technique, though it is worse than Fed-CVLC.
 \begin{itemize}
	  \item  Fed-CVLC can steadily achieve the highest model accuracy after every global iteration in all experiment cases under both IID and non-IID data distributions. The improvement of Fed-CVLC indicates the effectiveness to compress model updates with variable-length codes. 
	  \item Fed-CVLC is a generic design in that both PQ and QSGD can be applied to outperform corresponding baselines in terms of model accuracy. 
	  \item In most cases, combining PQ and  $Top_k$  can achieve better model accuracy than that of $Top_k$, indicating that the design combining quantization and sparsification is the state-of-the-art compression technique, though it is worse than Fed-CVLC. 
	  \item  The PQ algorithm is slightly better than the QSGD algorithm in terms of model accuracy. In particular, the performance of $Q_6Top_k$ is even worse than that of $Top_k$ because QSGD cannot work well if the number of bits for quantization is too few.  
\end{itemize}

\begin{table*}[htb]
\setlength{\abovecaptionskip}{-0.05cm}
	\centering
	\caption{Comparison of total uplink communication traffic (MB) consumed by different compression algorithms to reach  target model accuracy when using {\color{black}PQ(Top) and QSGD(Bottom)} for quantization. }
	\label{tab:TrafficPQ}
	\begin{tabular}{|c|c|c|c|c|c|c|}
		\hline
		& Fed-CVLC & $Top_k$ & $P_6Top_k$ & $P_8Top_k$ & $P_{10}Top_k$ & Reduced \% by ours \\
		\hline
		CIFAR-10+IID (63\%) &\textbf{7.15} & 11.44 & 8.58 & 8.58 & 9.30 & 16.67\%\\
		CIFAR-10+non-IID (60\%) & \textbf{5.72} & 11.44 & 8.58 & 9.30 & 8.58 & 33.33\%\\
		FEMNIST (80\%) & \textbf{5.01} & 10.73 & 8.58 & 8.58 & 9.30 & 41.61\%\\
		CIFAR-100+IID (46\%) & \textbf{128.75} & 205.99 & 167.37 & 167.37 & 167.37 & 23.07\%\\
		CIFAR-100+non-IID (44\%) & \textbf{128.75} & 231.74 & 180.24 & 167.37 & 180.24 & 23.07\%\\
		\hline
	\end{tabular}
	\\[4pt]
	% \caption{Comparison of total uplink communication traffic (MB) consumed by different compression algorithms to reach target model accuracy when using QSGD for quantization.}
	\label{tab:TrafficQSGD}
	\begin{tabular}{|c|c|c|c|c|c|c|}
		\hline
		& Fed-CVLC & $Top_k$ & $Q_8Top_k$ & $Q_{10}Top_k$ & $Q_{12}Top_k$ & Reduced \% by ours  \\
		\hline
		CIFAR-10+IID(63\%) & \textbf{7.15} & 11.44 & 8.58 & 10.01 & 9.30 & 16.67\%\\
		CIFAR-10+non-IID(60\%) & \textbf{5.72} & 11.44 & 8.58 & 8.58 & 8.58 & 33.33\%\\
		FEMNIST (80\%)  & \textbf{5.01} & 10.73 & 8.58 & 9.30 & 8.58 & 41.61\%\\
		CIFAR-100+IID(46\%) & \textbf{128.75} & 205.99 & 180.24 & 167.37 & 167.37 & 23.07\%\\
		CIFAR-100+non-IID(44\%) & \textbf{167.37} & 231.74 & 231.74 & 218.87 & 218.87 & 23.53\%\\
		\hline
	\end{tabular}
 \vspace{-2.0em}
\end{table*}

%\subsubsection{Comparing Communication Traffic}

% \begin{table*}[htb]
	%     \centering
	%     \caption{Comparison of total uplink communication traffic (MB) consumed by different compression algorithms to reach  the target model accuracy using QSGD for quantization.}
	%     \begin{tabular}{ccccc|c}
		%         \hline
		%          & Fed-CVLC & $Top_k$ & $Q_8Top_k$ & $Q_{10}Top_k$ & Reduced \% by ours  \\
		%         \hline
		%         CIFAR-10+IID(63\%) & \textbf{7.15} & 11.44 & 8.58 & 10.01 & 16.67\%\\
		%         CIFAR-10+non-IID(60\%) & \textbf{5.72} & 11.44 & 8.58 & 8.58 & 33.33\%\\
		%         FEMNIST (80\%)  & \textbf{5.01} & 10.73 & 8.58 & 9.30 & 41.61\%\\
		%         CIFAR-100+IID(46\%) & \textbf{128.75} & 205.99 & 180.24 & 167.37 & 23.07\%\\
		%         CIFAR-100+non-IID(44\%) & \textbf{167.37} & 231.74 & 231.74 & 218.87 & 23.53\%\\
		%         \hline
		%     \end{tabular}
	%     \label{tab:TrafficQSGD}
	% \end{table*}
\noindent{\bf Comparing Communication Traffic.} 
To verify that Fed-CVLC can significantly diminish communication cost, we further compare the total volume  of consumed uplink communication traffic by  different compression algorithms.  In other words, we accumulate the amount of uplink traffic until the trained model can reach the target model accuracy on test sets. %We set the target accuracy as 63\% for CIFAR-10+IID, 60\% for CIFAR-10+non-IID, 80\% for FEMNIST, 46\% for CIFAR-100+IID and 44\% for CIFAR-100+non-IID, respectively. 
The experiment results are presented in Table~\ref{tab:TrafficPQ}. %when using PQ for quantization. Due to limited space, we only compare Fed-CVLC and the second best baseline here, and the complete comparison with all baselines is presented in Appendix~\ref{QSGDResult}. 
Experiment results shed light on that: 1) Fed-CVLC always consumes the least amount of communication traffic to achieve the target model accuracy in all experiment cases. 2)  The last column in table shows the percent of communication traffic further shrunk by Fed-CVLC in comparison with the second best one. It shows that Fed-CVLC can further reduce the communication cost by 16.67\%-41.61\%  based on the state-of-the-art compression technique. 3) Model compression can sheer reduce communication cost. For example, for training the CNN model with 5 million parameters to classify CIFAR-100, 
the total uplink communication traffic is reduced to  128.75MB over the entire training process.

\section{Conclusion}
\label{Conclusion}

FL is a cutting-edge technique in preserving data privacy during the training of machine learning models. Different from traditional machine learning, FL will not relocate data samples, which however can incur heavy communication overhead. {Until now, quantization and sparsification are two most popular compression approaches to  prohibiting   communication cost of FL.} Yet, existing compression algorithms rigidly set a fixed code length for model compression, restricting the effectiveness of  compression. {In this work, we made an initial attempt  to expand  the compression  design space in FL by generalizing  quantization and sparsification with variable-length codes.} Not only we considered  practical implementation of communication packets in our design, but also conducted rigorous analysis to guarantee  the  performance of our algorithm. 
Comprehensive experiments have been conducted with public datasets to demonstrate that Fed-CVLC can considerably shrink clients' uplink communication traffic in FL. In comparison with the state-of-the-art technique, Fed-CVLC can further shrink communication traffic by 27.64\%  or improve model accuracy  by 3.21\% on average. Our future work is to extend Fed-CVLC to make it applicable in more advanced FL systems such as decentralized FL or hierarchical FL systems.

\appendices

\section{}
%\section{Proof of Lemmas and Theorems}
\subsection{Proof of Lemma~\ref{LEM:QuanError}}
\label{Proof_LEM:QuanError}
	We assume that a vector $\mathbf{\overline{U}}$ has $z_y$ elements, which are quantized by $y$ bits, to get vector $\widetilde{\mathbf{U}}$. Each element in vector $\widetilde{\mathbf{U}}$ is further scaled by dividing a factor $B$ to obtain vector $\hat{\mathbf{U}}$. The error between $\mathbf{\overline{U}}$ and $\hat{\mathbf{U}}$ is therefore bounded by:
	%We first discussed the error caused by vector $\mathbf{\overline{U}}$ after unbiased quantization to $\widetilde{\mathbf{U}}$, and then scaled to $\hat{\mathbf{U}}$ by constant $B$ .
	\begin{align}
		\mathbb{E}\|\hat{\mathbf{U}}-\mathbf{\overline{U}}\|^2&=\mathbb{E}\|\frac{\widetilde{\mathbf{U}}}{B}-\mathbf{\overline{U}}\|^2\notag\\
		%&=\mathbb{E}\left\|\frac{\widetilde{\mathbf{U}}-\mathbf{\overline{U}}}{B}-(1-\frac{1}{B})\mathbf{\overline{U}}\right\|^2\notag\\
		&\overset{(a)}{=}\mathbb{E}\|\frac{\widetilde{\mathbf{U}}-\mathbf{\overline{U}}}{B}\|^2+\mathbb{E}\|(1-\frac{1}{B})\mathbf{\overline{U}}\|^2\notag\\
		&\overset{(b)}{\le} \frac{Q(z_y, y)+(B-1)^2}{B^2}\|\mathbf{\overline{U}}\|^2,
	\end{align}
	where equality $(a)$ holds because the quantization algorithm is unbiased. %Then, we have $\mathbb{E}[\widetilde{\mathbf{U}}-\mathbf{\overline{U}}]=0$; 
 Inequality $(b)$ holds because $Q(z_y,y)$ represents the error by quantizing each element  with $y$ bits for  total $z_y$ elements, \emph{i.e.,}$\mathbb{E}\|\widetilde{\mathbf{U}}-\mathbf{\overline{U}}\|^2\le Q(z_y,y)\|\mathbf{\overline{U}}\|^2$.
 %In other words,  $Q(z_y,y)$ represents the quantization error, \emph{i.e.,}$\mathbb{E}\|\widetilde{\mathbf{U}}-\mathbf{\overline{U}}\|^2\le Q(z_y,y)\|\mathbf{\overline{U}}\|^2$.
 
\subsection{Proof of Proposition~\ref{LEM:ErrorBound}}
\label{Proof_LEM:ErrorBound}	We start from analyzing the compression error caused by variable-length coding.
	The error $\|\hat{\mathbf{U}}^i_t-\mathbf{U}^i_t\|^2$ has two components: sparsification error and  quantization error with variable-length codes for $Y$ different groups.  By substituting the power law distribution in Definition~\ref{DEF:PowerDist} into $U^i_t\{l\}$, %(Recall that $U^i_t\{l\}$ is the element with the $l$-th largest absolute value in vector $\mathbf{U}^i_t$), 
	we have
	\begin{align}
		&\mathbb{E}\frac{\|\hat{\mathbf{U}}^i_t-\mathbf{U}^i_t\|^2}{\|\mathbf{U}^i_t\|^2}=\mathbb{E}\frac{\sum_{y=1}^Y\|\frac{\widetilde{\mathbf{U}}^i_{ty}}{B}-\mathbf{U}^i_{ty}\|^2+\sum_{l=k+1}^d(U^i_t\{l\})^2}{\|\mathbf{U}^i_t\|^2}\notag\\
		&\le \frac{\sum_{y=1}^Y\frac{Q(z_y,y)+(B-1)^2}{B^2}\|\mathbf{U}^i_{ty}\|^2}{\|\mathbf{U}^i_t\|^2} + \frac{\sum_{l=k+1}^d(U^i_t\{l\})^2}{\|\mathbf{U}^i_t\|^2}\notag\\
		&\approx \frac{d^{2\alpha_i+1}-(k+1)^{2\alpha_i+1}}{d^{2\alpha_i+1}-1} + \frac{1}{d^{2\alpha_i+1}-1}\sum_{y=1}^Y\notag\\
		&\quad\big((\frac{Q(z_y,y)}{B^2}+\frac{1}{B_c^2}) (Z_y^{2\alpha_i+1} - Z_{y+1}^{2\alpha_i+1}) \big),
		%&=1 + \frac{\sum_{r=1}^R(\frac{Q(k_r,b_r)+(B-1)^2}{B^2}-1)\|\mathbf{U}^i_{tr}\|^2}{\|\mathbf{U}^i_t\|^2}\\
		%&\approx 1 + \frac{1}{d^{2\alpha+1}-1}\sum_{r=1}^R\Bigg((\frac{Q(k_r,b_r)+(B-1)^2}{B^2}-1)\\
		%&\quad \Big((\sum_{m=1}^r k_l)^{2\alpha+1} - (\sum_{m=1}^{r-1}k_l+1)^{2\alpha+1}\Big) \Bigg) ,
	\end{align}
	where $Z_{y+1}=z_{Y}+z_{Y-1}+\dots+z_{y+1}$ and $\mathbf{U}^i_{ty}$ is a $z_y$-dimension vector. Each element in $\mathbf{U}^i_{ty}$ will be  quantized with $y$ bits. The elements in the vector are $U^i_t\{l\}'s $ where $l=Z_{y+1}+1, Z_{y+1}+2, \dots, Z_{y+1}+z_y$. 
	To ensure that the trained model converges, we need to make the compression algorithm satisfy the condition that $\mathbb{E}\frac{\|\hat{\mathbf{U}}^i_t-\mathbf{U}^i_t\|^2}{\|\mathbf{U}^i_t\|^2}\le\gamma_i$,where $0<\gamma_i<1$. %The condition $0<\gamma_i<1$ 
 This conditions implies that the compression is bounded such that the trained model will not diverge in the end. 
	Therefore, we have $B>\max_{\forall y}\frac{Q(z_y,y)+1}{2}$ to guarantee  $0<\gamma_i<1$. 

\subsection{Proof of Lemma~\ref{LEM:ConvexOfF}}
\label{Proof_LEM:ConvexOfF}
Recall that $Z_r=P_1+\dots+P_r$ and $Z_{r-2}=P_1+\dots+P_{r-2}$.  $X$ is a fixed number, and thus $P_r+P_{r-1}=X$ is fixed. We can get $f(P_r)=(\frac{Q(P_r)}{B^2}+\frac{1}{B_c^2})\frac{Z_{r}^{\beta_i}-(Z_r-P_r)^{\beta_i}}{d^{\beta_i}-1} + (\frac{Q(X-P_r)}{B^2}+\frac{1}{B_c^2}) \frac{(Z_r-P_r)^{\beta_i}-Z_{r-2}^{\beta_i}}{d^{\beta_i}-1}.$
% \begin{align}
% 	f(X, P_{r-1}, P_r)&=(\frac{Q(P_r)}{B^2}+\frac{1}{B_c^2})\frac{Z_{r}^{\beta_i}-Z_{r-1}^{\beta_i}}{d^{\beta_i}-1}\notag\\
% 	&\quad + (\frac{Q(P_{r-1})}{B^2}+\frac{1}{B_c^2}) \frac{Z_{r-1}^{\beta_i}-Z_{r-2}^{\beta_i}}{d^{\beta_i}-1},
% \end{align}
% where $Z_r=P_1+\dots+P_r$ and $Z_{r-2}=P_1+\dots+P_{r-2}$.  $X$ is a fixed number, and thus $P_r+P_{r-1}=X$ is fixed. We can therefore transform the above equation into
% \begin{align}
% 	f(P_r)&=(\frac{Q(P_r)}{B^2}+\frac{1}{B_c^2})\frac{Z_{r}^{\beta_i}-(Z_r-P_r)^{\beta_i}}{d^{\beta_i}-1}\notag\\
% 	&\quad + (\frac{Q(X-P_r)}{B^2}+\frac{1}{B_c^2}) \frac{(Z_r-P_r)^{\beta_i}-Z_{r-2}^{\beta_i}}{d^{\beta_i}-1}.
% \end{align}
We derive %each of the four terms in 
the above equation to have:
\begin{align}
	f'(P_r)&=\underbrace{\frac{Q'(P_y)}{B^2}*\frac{Z_{r}^{\beta_i}-(Z_r-P_r)^{\beta_i}}{d^{\beta_i}-1}}_{A_1}\notag\\
	&\quad+\underbrace{(\frac{Q(P_r)}{B^2}+\frac{1}{B_c^2})\frac{\beta_i(Z_r-P_r)^{\beta_i-1}}{d^{\beta_i}-1}}_{A_2}\notag\\
	&\quad -\underbrace{\frac{Q'(X-P_r)}{B^2}*\frac{(Z_r-P_r)^{\beta_i}-Z_{r-2}^{\beta_i}}{d^{\beta_i}-1}}_{A_3}\notag\\
	&\quad-\underbrace{(\frac{Q(X-P_r)}{B^2}+\frac{1}{B_c^2})\frac{\beta_i(Z_r-P_r)^{\beta_i-1}}{d^{\beta_i}-1}}_{A_4}.
\end{align}
Here $Q'(P_r)$ is the derivative of the quantization error.
Since $Q(P_y)$ is an increasing strongly convex function, $Q'(P_y)$ is monotonically increasing. Moreover, the derivative of $\frac{Z_{r}^{\beta_i}-(Z_r-P_r)^{\beta_i}}{d^{\beta_i}-1}$  is a monotonically increasing function. Since the two terms multiplied in $A_1$ are greater than 0, $A_1$ is monotonically increasing. Similarly, we get that $A_3$ is monotonically decreasing. Therefore $A_1-A_3$ is monotonically increasing.
%We next analyze $A_2-A_4$, which is
%Note that we have $A_2-A_4=\frac{\beta_i(Z_r-P_r)^{\beta_i-1}}{B^2(d^{\beta_i}-1)}(Q(P_r)-Q(X-P_r)).$
Similar to the analysis of $A_1-A_3$, we have that  $A_2-A_4$ is monotonically increasing.
Since $A_1-A_3$ and $A_2-A_4$ are both monotonically increasing, it manifests that $f''(P_r) > 0$, and hence $f(P_r)$ is a strongly convex function.

We proceed to  discuss why $P_r > P_{r-1}$.
If $P_r<\frac{X}{2}$, it means $P_r<P_{r-1}$ and $P_r<X-P_{r}$. There are $\frac{Q'(P_y)}{B^2}<\frac{Q'(X-P_r)}{B^2}$ and $\frac{Z_{r}^{\beta_i}-(Z_r-P_r)^{\beta_i}}{d^{\beta_i}-1}<\frac{(Z_r-P_r)^{\beta_i}-Z_{r-2}^{\beta_i}}{d^{\beta_i}-1}$ in this scenario. Hence, we have $A_1-A_3<0$ and $A_2-A_4<0$, suggesting that $f(P_r)$ is a decreasing function if  $P_r<\frac{X}{2}$, and thus $f(P_r)$ can be decreased if we increase $P_r$ until $P_r>\frac{X}{2}$.
Thus, the optimal $P_r$ must be in the range $(\frac{X}{2}, X)$, implying that  $P_r>P_{r-1}$.

\subsection{Case Study}
\label{CaseStudy}
We show that quantization error $Q(P_r)$ is an increasing convex function by conducting a case study with two typical quantization algorithms,  namely PQ and QSGD algorithms.

In the PQ \cite{suresh2017distributed} algorithm,  centroids are uniformly distributed within the value range of  model updates.  Each model update within each interval %(defined by two consecutive centroids) 
is mapped to its upper or lower centroids. The quantization error of PQ is $Q(P_r)=\frac{P_r}{(2^{\frac{b-H}{P_r}-s}-1)^2}$ \cite{suresh2017distributed}.
% \begin{align}
	% \label{EQ:PQError}
	%     Q(P_r)=\frac{P_r}{(2^{\frac{b-H}{P_r}-s}-1)^2}.
	% \end{align}

\begin{theorem}
	\label{THE:ConvexOfQPQ}
	The quantization error $Q(P_r)$ of PQ is an increasing convex function with respect to $P_r$.
\end{theorem}
\begin{proof}
	%Considering $Q(P_r)=\frac{P_r}{(2^{\frac{b-H}{P_r}-s}-1)^2}$.
	We can obtain 
	$Q'(P_r)=\frac{2^{\frac{b-H}{P_r}-s+1}*\ln2*(b-H)}{(2^{\frac{b-H}{P_r}-s}-1)^3P_r}+\frac{1}{(2^{\frac{b-H}{P_r}-s}-1)^2}>0.$ %So $Q(P_r)$ is an increasing function.
	% So $Q(P_r)$ is an increasing function. We decompose the above equation to obtain
 %    $Q'(P_r)=\frac{1}{(2^{\frac{b-H}{P_r}-s}-1)^2}+\frac{2^{\frac{b-H}{P_r}-s+1}*\ln2*(b-H)}{(2^{\frac{b-H}{P_r}-s}-1)^3P_r}.$
	% \begin{align}
	% 	&\frac{2^{\frac{b-H}{P_r}-s}-1+2^{\frac{b-H}{P_r}-s+1}*\ln2*\frac{b-H}{P_r}}{(2^{\frac{b-H}{P_r}-s}-1)^3}\notag\\
	% 	&=\frac{1}{(2^{\frac{b-H}{P_r}-s}-1)^2}+\frac{2^{\frac{b-H}{P_r}-s+1}*\ln2*(b-H)}{(2^{\frac{b-H}{P_r}-s}-1)^3P_r}.
	% \end{align}
	The second term in $Q'(P_r)$ is increasing as $P_r$ increases, so we analyze the first term separately.  The numerator of the derivative of the first term is $3[2^{\frac{b-H}{P_r}-s}-1]^2*2^{\frac{b-H}{P_r}-s}*\ln2*\frac{b-H}{P_r}*2^{\frac{b-H}{P_r}-s+1}-2^{\frac{b-H}{P_r}-s+1}[2^{\frac{b-H}{P_r}-s}-1]^3*[\ln2*\frac{b-H}{P_r}+1]>0.$
	% \begin{align}
	% 	&3[2^{\frac{b-H}{P_r}-s}-1]^2*2^{\frac{b-H}{P_r}-s}*\ln2*\frac{b-H}{P_r}*2^{\frac{b-H}{P_r}-s+1}\notag\\
	% 	&\quad-2^{\frac{b-H}{P_r}-s+1}[2^{\frac{b-H}{P_r}-s}-1]^3*[\ln2*\frac{b-H}{P_r}+1]>0.\notag
	% 	% &>3[2^{\frac{b-H}{P_r}-s}-1]^3*\ln2*\frac{b-H}{P_r}*2^{\frac{b-H}{P_r}-s+1}\notag\\
	% 	% &\quad-2^{\frac{b-H}{P_r}-s+1}[2^{\frac{b-H}{P_r}-s}-1]^3*[\ln2*\frac{b-H}{P_r}+1]\notag\\
	% 	%&=[2^{\frac{b-H}{P_r}-s}-1]^3*2^{\frac{b-H}{P_r}-s+1}*\notag\\
	% 	% &\quad[3*\ln2*\frac{b-H}{P_r}-\ln2*\frac{b-H}{P_r}-1].
	% \end{align}
	This proves that $Q''(P_r) > 0$ and shows that $Q(P_r)$ is an increasing convex function.
\end{proof}
%Therefore we can substitute the compression error of PQ into $\mathbb{P}2$. The optimal transmission strategy is derived by solving the two given steps.
QSGD  is another typical quantization algorithm. In \cite{alistarh2017qsgd}, its  quantization error is $Q(P_r)=\min(Q_1(P_r), Q_2(P_r))$, where $Q_1(P_r) = \frac{P_r}{2^{2(\frac{b-H}{P_r}-s)}}$ and $Q_2(P_r)= \frac{\sqrt{P_r}}{2^{\frac{b-H}{P_r}-s}}$.
\begin{theorem}
	\label{THE:ConvexOfQQSGD}
	Both $Q_1(P_r)$ and $ Q_2(P_r) $  are increasing convex functions with respect to $P_r$. 
\end{theorem}

%\begin{proof}
%	Since $Q_1(P_r)$ and the error function of the PQ algorithm are basically the same, so their proof process is also the identical. We do not repeat them here.
%	
%	We perform a first-order derivative on $ Q_2(P_r) $ to obtain
%	\begin{align}
%		Q_2'(P_r)=\frac{P_r+\ln 4*(b-H)}{2^{\frac{b-H}{P_r}-s+1}x\sqrt{x}}>0.
%	\end{align}
%	So $Q_2(P_r)$ is an increasing function.
%	
%	We derive the numerator of the second-order derivative of $Q_2(P_r)$ as follows:
%	\begin{align}
%		&\frac{2^{\frac{b-H}{P_r}-s}}{\sqrt{P_r}}[2P_r^2-3P_r^2-3\ln 4*(b-H)*P_r+\ln 4*(b-H)*P_r+(\ln 4*(b-H))^2]\notag\\
%		&=\frac{2^{\frac{b-H}{P_r}-s}}{\sqrt{P_r}}[2(\ln 4*(b-H))^2-(P_r+\ln 4*(b-H))^2].
%	\end{align}
%	Thus we can know that $Q_2(P_r)$ is a convex function when $\frac{b-H}{P_r}>\frac{\sqrt{2}-1}{\ln 4}\approx1.742$. Consider that $\frac{b-H}{P_r}$ represents the size of each parameter in a packet, consisting of the position ID and the ID of the centroids. Therefore we can obtain that $\frac{b-H}{P_r}>\log_2d$, which is usually satisfying the aforementioned constraint. This proves that $Q_2(P_r)$ is also an increasing convex function.
%\end{proof}

The proof is similar to that  of Theorem~\ref{THE:ConvexOfQPQ}, and thus omitted.
%Therefore, we can substitute the two terms of the min function into $\mathbb{P}2$ to solve for the transmission policy and finally choose the optimal transmission.

\subsection{Proof of Theorem~\ref{THE:OptimalSolution}}
\label{Proof_THE:OptimalSolution}
We will prove Theorem~\ref{THE:OptimalSolution} by a contradiction. 
Considering that $f(X,P_{r-1},P_{r})$ is a strongly convex function, there is only one optimal solution to minimize $f(X,P_{r-1},P_{r})$ for $\forall r$. Thus, the sequence $P_1^*,\dots,P_R^*$ that minimizes $f(X,P_{r-1},P_{r})$ for $2\le r\le R$ is unique. We use $\gamma_i^*$ to denote the compression error in this case.
If there exists another sequence $P_1',\dots, P_R'$ and its corresponding error is $\gamma_i'<\gamma_i^*$. There must be a pair of $P_{r-1}'$ and $P_{r}'$ in this sequence that cannot minimize the function $f(X,P_{r-1},P_{r})$. Therefore, the number of model updates encapsulated into these two packets can be adjusted to obtain $P_{r-1}''$ and $P_{r}''$ so that $f(X,P_{r-1},P_{r})$ is further reduced to achieve a lower error  $\gamma_i''<\gamma_i'$. The sequence of $\gamma_i''$ continues to adjust the assignment of each two adjacent packets to gradually reduce the error, which will eventually result in $\gamma_i''<\gamma_i'$, contradicting the assumption that $\gamma_i'$ is the minimum value.

Thus, $P_1^*,\dots,P_R^*$ that minimize the $f$ function for any two adjacent packets can achieve the minimum error $\gamma_i^*$.

\clearpage

\bibliographystyle{IEEEtran} 
\bibliography{reference}

% Generated by IEEEtran.bst, version: 1.14 (2015/08/26)
\begin{thebibliography}{10}
\providecommand{\url}[1]{#1}
\csname url@samestyle\endcsname
\providecommand{\newblock}{\relax}
\providecommand{\bibinfo}[2]{#2}
\providecommand{\BIBentrySTDinterwordspacing}{\spaceskip=0pt\relax}
\providecommand{\BIBentryALTinterwordstretchfactor}{4}
\providecommand{\BIBentryALTinterwordspacing}{\spaceskip=\fontdimen2\font plus
\BIBentryALTinterwordstretchfactor\fontdimen3\font minus
  \fontdimen4\font\relax}
\providecommand{\BIBforeignlanguage}[2]{{%
\expandafter\ifx\csname l@#1\endcsname\relax
\typeout{** WARNING: IEEEtran.bst: No hyphenation pattern has been}%
\typeout{** loaded for the language `#1'. Using the pattern for}%
\typeout{** the default language instead.}%
\else
\language=\csname l@#1\endcsname
\fi
#2}}
\providecommand{\BIBdecl}{\relax}
\BIBdecl

\bibitem{yang2019federated}
Q.~Yang, Y.~Liu, T.~Chen, and Y.~Tong, ``{Federated machine learning: Concept
  and applications},'' \emph{ACM Transactions on Intelligent Systems and
  Technology (TIST)}, vol.~10, no.~2, pp. 1--19, 2019.

\bibitem{mcmahan2017communication}
B.~McMahan, E.~Moore, D.~Ramage, S.~Hampson, and B.~A. y~Arcas,
  ``{Communication-efficient learning of deep networks from decentralized
  data},'' in \emph{Artificial Intelligence and Statistics (AISTATS)}, 2017,
  pp. 1273--1282.

\bibitem{lim2020federated}
W.~Y.~B. Lim, N.~C. Luong, D.~T. Hoang, Y.~Jiao, Y.-C. Liang, Q.~Yang,
  D.~Niyato, and C.~Miao, ``{Federated learning in mobile edge networks: A
  comprehensive survey},'' \emph{IEEE Communications Surveys \& Tutorials
  (COMST)}, vol.~22, no.~3, pp. 2031--2063, 2020.

\bibitem{8466361}
P.~Wang, F.~Ye, and X.~Chen, ``{A Smart Home Gateway Platform for Data
  Collection and Awareness},'' \emph{IEEE Communications Magazine (MCOM)},
  vol.~56, no.~9, pp. 87--93, 2018.

\bibitem{hamer2020fedboost}
J.~Hamer, M.~Mohri, and A.~T. Suresh, ``{Fedboost: A communication-efficient
  algorithm for federated learning},'' in \emph{International Conference on
  Machine Learning (ICML)}, 2020, pp. 3973--3983.

\bibitem{rothchild2020fetchsgd}
D.~Rothchild, A.~Panda, E.~Ullah, N.~Ivkin, I.~Stoica, V.~Braverman,
  J.~Gonzalez, and R.~Arora, ``{Fetchsgd: Communication-efficient federated
  learning with sketching},'' in \emph{International Conference on Machine
  Learning (ICML)}, 2020, pp. 8253--8265.

\bibitem{konevcny2016federated}
J.~Kone{\v{c}}n{\`y}, H.~B. McMahan, F.~X. Yu, P.~Richt{\'a}rik, A.~T. Suresh,
  and D.~Bacon, ``{Federated learning: Strategies for improving communication
  efficiency},'' in \emph{Annual Conference on Neural Information Processing
  Systems (NeurIPS) Workshop on Private Multi-Party Machine Learning}, 2016,
  pp. 1--5.

\bibitem{he2016deep}
K.~He, X.~Zhang, S.~Ren, and J.~Sun, ``{Deep residual learning for image
  recognition},'' in \emph{IEEE conference on computer vision and pattern
  recognition (CVPR)}, 2016, pp. 770--778.

\bibitem{NeurIPS2017_3f5ee243}
A.~Vaswani, N.~Shazeer, N.~Parmar, J.~Uszkoreit, L.~Jones, A.~N. Gomez, L.~u.
  Kaiser, and I.~Polosukhin, ``{Attention is All you Need},'' in \emph{Advances
  in Neural Information Processing Systems (NeurIPS)}, vol.~30, 2017, pp.
  1--11.

\bibitem{wen2017terngrad}
W.~Wen, C.~Xu, F.~Yan, C.~Wu, Y.~Wang, Y.~Chen, and H.~Li, ``{TernGrad: Ternary
  Gradients to Reduce Communication in Distributed Deep Learning},'' in
  \emph{Annual Conference on Neural Information Processing Systems (NeurIPS)},
  2017, pp. 1509--1519.

\bibitem{bernstein2018signsgd}
J.~Bernstein, Y.-X. Wang, K.~Azizzadenesheli, and A.~Anandkumar, ``{signSGD:
  Compressed optimisation for non-convex problems},'' in \emph{International
  Conference on Machine Learning (ICML)}, 2018, pp. 560--569.

\bibitem{lin2017deep}
Y.~Lin, S.~Han, H.~Mao, Y.~Wang, and B.~Dally, ``{Deep Gradient Compression:
  Reducing the Communication Bandwidth for Distributed Training},'' in
  \emph{International Conference on Learning Representations (ICLR)}, 2018, pp.
  1--14.

\bibitem{sattler2019robust}
F.~Sattler, S.~Wiedemann, K.-R. M{\"u}ller, and W.~Samek, ``{Robust and
  Communication-Efficient Federated Learning From Non-i.i.d. Data},''
  \emph{{IEEE Transactions on Neural Networks and Learning Systems (TNNLS)}},
  vol.~31, no.~9, pp. 3400--3413, 2020.

\bibitem{m2021efficient}
A.~M~Abdelmoniem, A.~Elzanaty, M.-S. Alouini, and M.~Canini, ``{An efficient
  statistical-based gradient compression technique for distributed training
  systems},'' \emph{Machine Learning and Systems (MLSys)}, vol.~3, pp.
  297--322, 2021.

\bibitem{li2019convergence}
X.~Li, K.~Huang, W.~Yang, S.~Wang, and Z.~Zhang, ``{On the Convergence of
  FedAvg on Non-IID Data},'' in \emph{International Conference on Learning
  Representations (ICLR)}, 2020, pp. 1--26.

\bibitem{yang2021achieving}
H.~Yang, M.~Fang, and J.~Liu, ``{Achieving Linear Speedup with Partial Worker
  Participation in Non-IID Federated Learning},'' in \emph{International
  Conference on Learning Representations (ICLR)}, 2021, pp. 1--23.

\bibitem{9796724}
S.~Chen and B.~Li, ``{Towards Optimal Multi-Modal Federated Learning on Non-IID
  Data with Hierarchical Gradient Blending},'' in \emph{IEEE Conference on
  Computer Communications (INFOCOM)}, 2022, pp. 1469--1478.

\bibitem{9796719}
Z.~Wang, Y.~Zhu, D.~Wang, and Z.~Han, ``Fedfpm: A unified federated analytics
  framework for collaborative frequent pattern mining,'' in \emph{IEEE
  Conference on Computer Communications (INFOCOM)}, 2022, pp. 61--70.

\bibitem{9796721}
Y.~Liu, L.~Xu, X.~Yuan, C.~Wang, and B.~Li, ``The right to be forgotten in
  federated learning: An efficient realization with rapid retraining,'' in
  \emph{IEEE Conference on Computer Communications (INFOCOM)}, 2022, pp.
  1749--1758.

\bibitem{9796818}
J.~Perazzone, S.~Wang, M.~Ji, and K.~S. Chan, ``{Communication-Efficient Device
  Scheduling for Federated Learning Using Stochastic Optimization},'' in
  \emph{IEEE Conference on Computer Communications (INFOCOM)}, 2022, pp.
  1449--1458.

\bibitem{9796935}
B.~Luo, W.~Xiao, S.~Wang, J.~Huang, and L.~Tassiulas, ``{Tackling System and
  Statistical Heterogeneity for Federated Learning with Adaptive Client
  Sampling},'' in \emph{IEEE Conference on Computer Communications (INFOCOM)},
  2022, pp. 1739--1748.

\bibitem{9488756}
Z.~Wang, H.~Xu, J.~Liu, H.~Huang, C.~Qiao, and Y.~Zhao, ``{Resource-Efficient
  Federated Learning with Hierarchical Aggregation in Edge Computing},'' in
  \emph{IEEE Conference on Computer Communications (INFOCOM)}, 2021, pp. 1--10.

\bibitem{suresh2017distributed}
A.~T. Suresh, X.~Y. Felix, S.~Kumar, and H.~B. McMahan, ``{Distributed mean
  estimation with limited communication},'' in \emph{International Conference
  on Machine Learning (ICML)}, 2017, pp. 3329--3337.

\bibitem{alistarh2017qsgd}
D.~Alistarh, D.~Grubic, J.~Li, R.~Tomioka, and M.~Vojnovic, ``{QSGD:
  Communication-efficient SGD via gradient quantization and encoding},'' in
  \emph{Advances in Neural Information Processing Systems (NeurIPS)}, 2017, pp.
  1709--1720.

\bibitem{honig2022dadaquant}
R.~H{\"o}nig, Y.~Zhao, and R.~Mullins, ``{DAdaQuant: Doubly-adaptive
  quantization for communication-efficient Federated Learning},'' in
  \emph{International Conference on Machine Learning (ICML)}, 2022, pp.
  8852--8866.

\bibitem{wangni2018gradient}
J.~Wangni, J.~Wang, J.~Liu, and T.~Zhang, ``{Gradient sparsification for
  communication-efficient distributed optimization},'' in \emph{Advances in
  Neural Information Processing Systems (NeurIPS)}, 2018, pp. 1299--1309.

\bibitem{stich2018sparsified}
S.~U. Stich, J.-B. Cordonnier, and M.~Jaggi, ``{Sparsified SGD with memory},''
  in \emph{Advances in Neural Information Processing Systems (NeurIPS)}, 2018,
  pp. 4447--4458.

\bibitem{qian2021error}
X.~Qian, P.~Richt{\'a}rik, and T.~Zhang, ``{Error compensated distributed SGD
  can be accelerated},'' \emph{Advances in Neural Information Processing
  Systems (NeurIPS)}, vol.~34, pp. 30\,401--30\,413, 2021.

\bibitem{NEURIPS2019_d202ed5b}
D.~Basu, D.~Data, C.~Karakus, and S.~Diggavi, ``{Qsparse-local-SGD: Distributed
  SGD with Quantization, Sparsification and Local Computations},'' in
  \emph{Advances in Neural Information Processing Systems (NeurIPS)},
  H.~Wallach, H.~Larochelle, A.~Beygelzimer, F.~d\textquotesingle
  Alch\'{e}-Buc, E.~Fox, and R.~Garnett, Eds., vol.~32, 2019, pp. 1--12.

\bibitem{wang2019adaptive}
S.~Wang, T.~Tuor, T.~Salonidis, K.~K. Leung, C.~Makaya, T.~He, and K.~Chan,
  ``{Adaptive federated learning in resource constrained edge computing
  systems},'' \emph{IEEE Journal on Selected Areas in Communications (JSAC)},
  vol.~37, no.~6, pp. 1205--1221, 2019.

\bibitem{luo2020cost}
B.~Luo, X.~Li, S.~Wang, J.~Huang, and L.~Tassiulas, ``{Cost-Effective Federated
  Learning Design},'' in \emph{International Conference on Computer
  Communications (INFOCOM)}.\hskip 1em plus 0.5em minus 0.4em\relax IEEE, 2021,
  pp. 1--10.

\bibitem{9589061}
H.~Yang, J.~Liu, and E.~S. Bentley, ``{CFedAvg: Achieving Efficient
  Communication and Fast Convergence in Non-IID Federated Learning},'' in
  \emph{International Symposium on Modeling and Optimization in Mobile, Ad hoc,
  and Wireless Networks (WiOpt)}, 2021, pp. 1--8.

\bibitem{li2020talk}
L.~Li, D.~Shi, R.~Hou, H.~Li, M.~Pan, and Z.~Han, ``{To Talk or to Work:
  Flexible Communication Compression for Energy Efficient Federated Learning
  over Heterogeneous Mobile Edge Devices},'' in \emph{IEEE International
  Conference on Computer Communications (INFOCOM)}, 2021, pp. 1--10.

\bibitem{islam2016quality}
N.~Islam, C.~C. Bawn, J.~Hasan, A.~I. Swapna, and M.~S. Rahman, ``{Quality of
  service analysis of Ethernet network based on packet size},'' \emph{Journal
  of Computer and communications (JCC)}, vol.~4, no.~4, pp. 63--72, 2016.

\bibitem{platt1998sequential}
J.~Platt, ``{Sequential Minimal Optimization: A Fast Algorithm for Training
  Support Vector Machines},'' Microsoft, Tech. Rep. MSR-TR-98-14, April 1998.

\bibitem{wang2020optimizing}
H.~Wang, Z.~Kaplan, D.~Niu, and B.~Li, ``{Optimizing Federated Learning on
  Non-IID Data with Reinforcement Learning},'' in \emph{International
  Conference on Computer Communications (INFOCOM)}.\hskip 1em plus 0.5em minus
  0.4em\relax IEEE, 2020, pp. 1698--1707.

\end{thebibliography}

\end{document}